\documentclass[sigconf]{IEEEtran}
\usepackage{amsmath,amssymb,amsfonts,amsthm}
\usepackage{graphicx}
\usepackage{subcaption} 
\usepackage{textcomp}
\usepackage{xcolor}
\usepackage{epigraph}

\usepackage{enumitem}
\usepackage{times}
\usepackage{soul}
\usepackage{url}
\usepackage{bbm}
\usepackage{booktabs}
\usepackage{algorithmic}
\usepackage{mathrsfs}
\usepackage{multirow}
\usepackage{makecell}

\usepackage{bm}
\usepackage[boxed,ruled]{algorithm2e}
\usepackage{wrapfig}
\usepackage{tcolorbox}

\usepackage{ulem}

\usepackage{titlesec}

\def\BibTeX{{\rm B\kern-.05em{\sc i\kern-.025em b}\kern-.08em
    T\kern-.1667em\lower.7ex\hbox{E}\kern-.125emX}}
    
\newtheorem{thm}{Theorem}
\newtheorem{defn}[thm]{Definition}

\usepackage{subcaption}

\begin{document}

\title{Multi-objective Explanations of GNN Predictions}

\author{
\IEEEauthorblockN{
Yifei Liu\IEEEauthorrefmark{1}, 
Chao Chen\IEEEauthorrefmark{2},
Yazheng Liu\IEEEauthorrefmark{1}, 
Xi Zhang\IEEEauthorrefmark{1}, 
Sihong Xie\IEEEauthorrefmark{2}
}\\
\IEEEauthorblockA{
\footnotesize
\IEEEauthorrefmark{1}Key Laboratory of Trustworthy Distributed Computing and Service (MoE), BUPT
\IEEEauthorrefmark{2}Computer Science and Engineering Dept, Lehigh University \\
liuyifei@bupt.edu.cn,
chc517@lehigh.edu,
\{LiuYZ,zhangx\}@bupt.edu.cn,
xiesihong1@gmail.com
}

}

\maketitle

\begin{abstract}
Graph Neural Network (GNN) has achieved state-of-the-art performance in various high-stake prediction tasks,
but multiple layers of aggregations on graphs with irregular structures make GNN a less interpretable model.
Prior methods use simpler subgraphs
to simulate the full model,
or counterfactuals to identify the causes of a prediction.
The two families of approaches aim at two distinct objectives, ``simulatability'' and ``counterfactual relevance'',
but it is not clear
how the objectives can jointly influence the human understanding of an explanation.
We design a user-study to investigate such joint effects,
and use the findings to design
a multi-objective optimization (MOO) algorithm to find Pareto optimal explanations that are well-balanced in simulatability and counterfactual.
Since the target model can be of any GNN variants and may not be accessible due to privacy concerns,
we design a search algorithm using zero-th order information without accessing the architecture and parameters of the target model.
Quantitative experiments on nine graphs from four applications demonstrate that the Pareto efficient explanations dominate single-objective baselines that use first-order continuous optimization or discrete combinatorial search.
The explanations are further evaluated in robustness and sensitivity to show their capability of revealing convincing causes, while being cautious about the possible confounders.
The diverse dominating counterfactuals
can certify the feasibility of algorithmic recourse,
that can potentially promote algorithmic fairness where humans are participating in the decision-making
using GNN.
\end{abstract}


\section{Introduction}
Graphs represent relations between entities and have been used to model social networks ~\cite{tang2009relational}, biological networks~\cite{Marinka2017}, and online reviews~\cite{rayana2015collective}.
On prediction tasks on graphs,
such as node classification, link prediction, and graph classification~\cite{kipf2017gcn,chen2018fastgcn,velivckovic2018gat,hamilton2017graphsage}, GNN exploits the relations to aggregate information in a neighborhood of each node to achieve state-of-the-art predictive performance.
However, the aggregations over many nodes multi-hops away
make the GNN predictions too opaque to be understood and trusted by humans.
Explanations of the GNN predictions try to simplify the computation to deliver societal merits, such as justifying the predictions, fulfilling legal regulation~\cite{Goodman2017},
and algorithmic recourse~\cite{Ustun2019,Russell2019,Barocas2020fat}.
For example,
when warning an online shopper about frauds detected using GNN on a review graph~\cite{rayana2015collective},
the user may ask ``why I am a victim of frauds'' and
expect an explanation such as ``the website you're viewing has connections with certain suspicious IP addresses''.
We focus on explaining GNN predictions made on graph nodes.
\begin{table}[b]
\caption{\small Prior explanation methods \textit{v.s}. the proposed method.}
    \centering
    \vspace{-.1in}
    \footnotesize
    \begin{tabular}{c||ccc|cc|ccc|c}\toprule
                    \parbox{1.7cm}{
                    $\rightarrow$Algorithms\\ 
                    $\downarrow$ Properties} & \rotatebox[origin=c]{90}{Diverse-CF
                    }
                    & \rotatebox[origin=c]{90}{Recourse
                    }
                    & \rotatebox[origin=c]{90}{Convex
                    }
                    & \rotatebox[origin=c]{90}{DeepLIFT
                    }
                    & \rotatebox[origin=c]{90}{LIME 
                    }
                    & \rotatebox[origin=c]{90}{GNNLIME}
                    & \rotatebox[origin=c]{90}{GNNExplainer}
                    & \rotatebox[origin=c]{90}{Attention}
                    & \rotatebox[origin=c]{90}{The proposed} \\\hline
        Surrogate & & & & & $\ast$ & $\ast$ & $\ast$ & $\ast$ & \\
        Gradient  & $\ast$ & & & $\ast$ & & & & &  \\
        Search & & $\ast$ & $\ast$ & & & & & & $\ast$  \\\hline
        Simulatability & $\ast$ & $\ast$ & $\ast$ & $\ast$ & $\ast$ &  & $\ast$ & $\ast$ & $\ast$\\
        Counterfactual & $\ast$ & $\ast$ & & $\ast$ & $\ast$ &  & $\ast$ & & $\ast$\\
        Blackbox & $\ast$ & $\ast$ & & & & & & & $\ast$\\\hline
        Human Evaluation & & & & & & & & & $\ast$ \\\hline
    \end{tabular}
\label{tab:differences}
\end{table}

According to~\cite{Lewis1986},
``\textit{To explain an event is to provide some information about its causal history}'' and the explanation of a prediction can be defined in two ways.
First, an explanation can be a causal chain consisting of a forward mapping from inputs and model parameters (the ``causes'' $X$) to model prediction (the ``outcome'' $Y$) via steps of computations.
An explanation with good simulatability would allow humans to more easily forward simulate the causal chain (possibly a simplified version).
Second, an explanation can be in the form of counterfactuals~\cite{Miller2019}: an event $X$ is said to have caused event $Y$,
if in the counterfactual where $X$ did not happen, $Y$ would not have happened.
Counterfactuals allow humans to see the impact of $X$ on $Y$,
and
prior works show that humans do counterfactual reasoning in their day-to-day life~\cite{Miller2019,Binns2018chi}.
We define counterfactual relevance as the amount of change in the probability that $Y$ happens when the cause $X$ is altered.
There are additional desiderata. 
To give humans a better sense of causal relationship,
an explanation of an outcome $Y$ should be robust to perturbations irrelevant to the cause $X$ but sensitive to changes in $X$.
Diverse counterfactuals for algorithmic recourse
with minimal changes in the decision subjects,
e.g., a human on a dating site,
allow human agency in the decision-making~\cite{Ustun2019}.


Explaining GNN is gaining more attention, and yet there is no study of the interactions between the two metrics, simulatability and counterfactual relevance, from the human and computation perspectives.
Gradient-based methods~\cite{baldassarre2019explainability,pope2019explainability} use magnitudes of
gradients to highlight important edges or node features.
Such methods aim at counterfactuals since the gradients indicate how fast the prediction (the ``outcome'') changes with respect to small perturbations in the highlighted input (the ``cause'').
Learning-based explanation methods,
including GNNExplainer~\cite{ying2019gnn} and GNNLIME~\cite{graphlime},
extract a simpler surrogate model to faithfully approximate a GNN prediction and thus promote simulatability,
without concerning counterfactual relevance.
Explanation methods based on gradients~\cite{ying2019gnn} need to access the target model as a whitebox and may break privacy and security constraints. 
See Table~\ref{tab:differences} and related work for comparisons.


Inspired by the two modes of human thinking studied in psychology~\cite{kahneman2011thinking},
we hypothesize that human perception of an explanation
is a function of both metrics.
We conjecture a cognitive process where humans first intuitively make sense of the outcome in a lightweight forward simulation using an explanation (System 1),
and then perform more effortful counterfactual reasoning (System 2) to figure out a cause of the outcome.
If the explanation is rejected due to low simulatability in the first phase,
humans will be less willing to seek for the causes.
Fig.~\ref{fig:two_dim} shows a pictorial representation of the hypothesis,
with four categories of explanations.
Gradient-based explanations are in the high counterfactual relevance, low simulatability category (region \textbf{A}),
and GNNExplainer has no guarantee of high counterfactual relevance but aims to achieve high simulatability (region \textbf{D}).
Table~\ref{tab:overall} in Section~\ref{sec:experiments} shows the quantitative evaluation of these methods.

\begin{figure}[t]
\centering
	\includegraphics[width=0.4\textwidth]{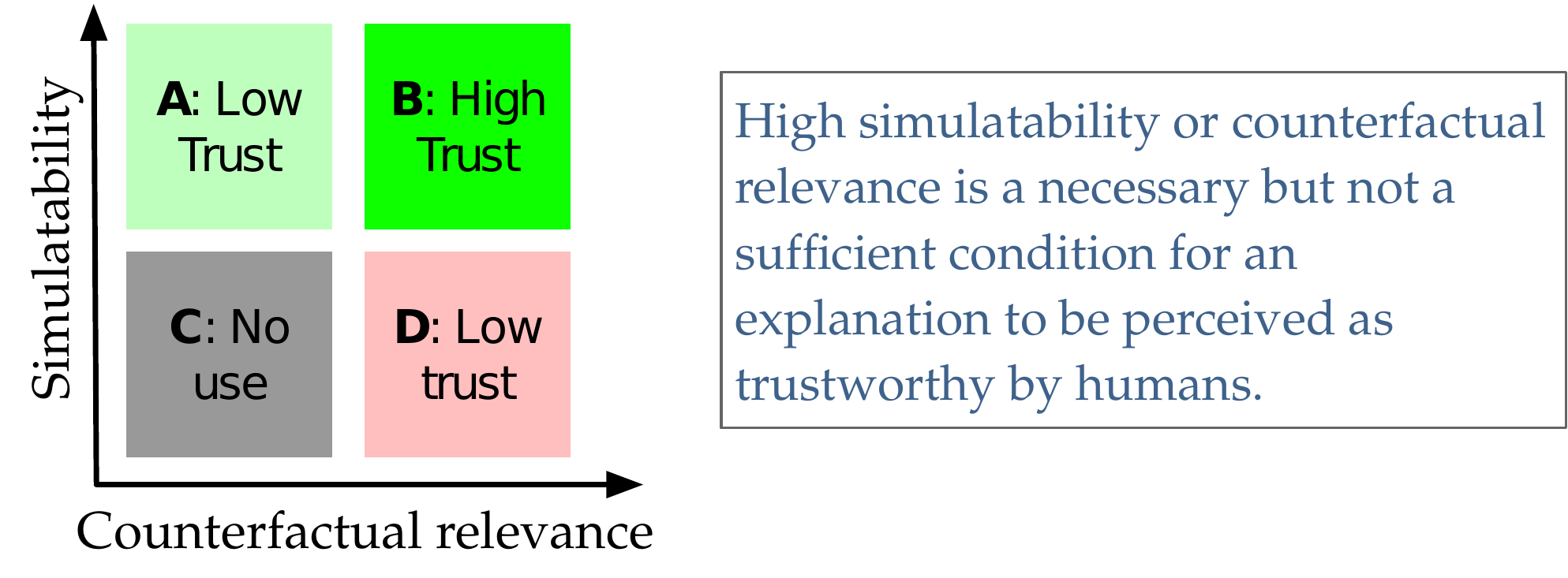}
		\caption{\small Simulatability and counterfactual relevance interact.}
\label{fig:two_dim}
\vspace{-.1in}
\end{figure}

To test the above hypotheses on GNN,
we adopt simple (small, acyclic, and connected)
subgraphs as explanations for forward simulation of a GNN prediction on a node.
Each explanation is associated with a counterfactual explanation that has some elements removed from the explanation to flip the prediction.
We generate explanations and counterfactuals in the four categories shown in Fig.~\ref{fig:two_dim}
and measure how simulatability and counterfactual relevance interact to influence human perception of the explanations.
Statistical analyses show that:
1) a low simulatability can, but not always,
prevent the adoption of an explanation,
making counterfactual reasoning less relevant.
2) conditioned on a high simulatability, high counterfactual relevance improves human acceptance of the explanation.

Given the joint effect of the two metrics on humans, 
current methods do not jointly maximize simulatability and counterfactual relevance.
Since the two metrics can be competing and trade-offs are necessary, we define Pareto efficient explanations and
formulate a multi-objective optimization problem to model the trade-offs.
Since the target model is a blackbox,
we design a depth-first search algorithm that accesses the zero-th order information of the model, i.e., the predictions, to identify Pareto efficient subgraph explanations.
Explanation search algorithms, such as those based on (mixed) integer programming~\cite{Ustun2019,Russell2019} and subgraph enumeration~\cite{Yoshida2019kdd},
employ similar searches, and yet they are single-objective optimization.
Though less expensive,
gradient-based approaches~\cite{ying2018graph} are white-box methods and only find node/edge importance, while the generation of connected graphs still requires exhaustive search.
Further, we provide an analysis on the lack of robustness of gradient-based GNN explanations.
In the contrast,
we empirically verify the robustness and sensitivity of the optimal explaining subgraphs found by the proposed algorithm.
Although we strike for causal explanations,
we are cautious and formulate GNN using Structural Equation Model (SEM)
to prove that confounders can exist in a subgraph explanation and users must be cautioned that the found counterfactuals are not ``the'' causes of the predictions.
Lastly, we extensively verified that the proposed algorithm dominates single-objective baselines in both metrics  
on 9 datasets.


\section{Problem Formulation}
Assume that we have a GNN of $L$ layers trained to predict class distributions of the nodes on a graph $G=(V, E)$,
where $V$ is the set of nodes and $E$ is the set of edges connecting the nodes.
Let $\mathcal{N}(v_i)$ be the set of neighbors of $v_i\in V$.
On layer $l$, $l=1,\dots, L$ and for any node $v_i$, $i=1,\dots, |V|$,
GNN computes $\mathbf{h}_i^{(l)}$ using messages sent from $\mathcal{N}(v_i)$ to $v_i$, by the following operations:
\begin{eqnarray}
\mathbf{m}^{(l)}_{ji} & = & \textnormal{MSG}\left(\mathbf{h}^{(l-1)}_j,\mathbf{h}^{(l-1)}_i\right),\label{eq:message}\\
\mathbf{a}^{(l)}_i & = & \textnormal{AGG}\left(\left\{\mathbf{m}^{(l)}_{ji}|v_j \in \mathcal{N}(v_i)\right\}\right),\label{eq:aggregate}\\
\mathbf{h}^{(l)}_i & = & \textnormal{UPDATE}\left(\mathbf{a}^{(l)}_i, \boldsymbol{\theta}^{(l)}\right).
\label{eq:hidden}
\end{eqnarray}
The MSG function computes the message vector sent from $v_j$ to $v_i$ (e.g., $\mathbf{m}_{ji}^{(l)}=\mathbf{h}_{j}^{(l-1)}$).
The AGG function aggregates the messages sent from all $v_j\in \mathcal{N}(v_i)$ to $v_i$ and can be the element-wise sum, average, or maximum of the messages.
The UPDATE function uses parameter $\boldsymbol{\theta}^{(l)}$ to map $\mathbf{a}_i^{(l)}$ to $\mathbf{h}_i^{(l)}$.
One example is $\mathbf{h}_i^{(l)}=(\boldsymbol{\theta}^{(l)})^{\top} \mathbf{a}^{(l)}_i$, followed by some non-linear mapping such as ReLU.
The input node feature vector $\mathbf{x}_i$ for $v_i$ is regarded as $\mathbf{h}^{(0)}_i$.
The output of the GNN on node $v_i$ is $\mathbf{h}_i^{(L)}$, which can be softmaxed to the node class distribution $\mathbf{y}_i$ (a vector of class probabilities).
The parameters of GNN, $\boldsymbol{\theta}^{(l)}$, $l=1,\dots,L$,
are trained end-to-end on labeled nodes on $G$.
We define an explanation of the prediction $\mathbf{y}_i$ to be a subgraph $G_i$
of $G$ that contains the target node $v_i$~\cite{ying2019gnn}.
Besides being agnostic to the above details of architecture and parameters,
we desire the following properties of the explanations.

\noindent\textbf{Simulatability.}
A comprehensible explanation should be simulatable, defined by the following two aspects.
The simplicity of an explanation is related to the limit of human cognitive bandwidth~\cite{Miller1956} and sparsity is used as a proxy of simplicity~\cite{Du2019,Guidotti2018,ying2019gnn}.
We say that the explaining subgraph $G_i$ is $C$-sparse if $G_i$ contains no more than $C$ nodes. 
Due to the sparsity, $G_i$ does not allow full computation taken on the full graph $G$, and
the \textit{faithfulness} of $G_i$ measures how much the $G_i$ can reproduce $\mathbf{y}_i$ generated on $G$.
Similar to~\cite{Suermondt1992},
we measure faithfulness using the symmetric KL-divergence between the prediction $\mathbf{y}_i$ on $G$ and $\mathbf{y}_i^\prime$ on $G_i$ (the larger, the better):
\begin{equation}
\label{eq:sim}
\nu(G_i) = -(\textnormal{KL}(\mathbf{y}_i||\mathbf{y}^\prime_i) + \textnormal{KL}(\mathbf{y}_i^\prime||\mathbf{y}_i))\leq 0.
\end{equation}

\label{sec:qualitative}

\noindent\textbf{Counterfactual relevance.}
Let the above-defined subgraph $G_i$ be a ``fact''.
A counterfactual $\tilde{G}_i$ of $G_i$ is a perturbation of $G_i$.
We restrict the counterfactual to be a strict subgraph of $G_i$.
Let the difference between $G_i$ and $\tilde{G}_i$ be denoted by $\Delta(G_i,\tilde{G}_i)$,
the size of which is represented by $|\Delta(G_i,\tilde{G}_i)|$,
so that ${G}_i=\tilde{G}_i+\Delta(G_i, \tilde{G}_i)$
means adding $\Delta(G_i,\tilde{G}_i)$ to $\tilde{G}_i$ reconstructs $G_i$.
The class distributions of $v_i$ generated by the target GNN model on $G_i$ and $\tilde{G}_i$ are denoted by $\mathbf{y}_i^\prime$ and $\tilde{\mathbf{y}}_i^\prime$, respectively.
We define the counterfactual relevance~\cite{Miller2019} of the tuple $(G_i, \tilde{G}_i)$ when explaining $\mathbf{y}_i$ as
\begin{equation}
\label{eq:cfr}
\mu(G_i, \tilde{G}_i) = \frac{1}{|\Delta(G_i,\tilde{G}_i)|} (\nu(G_i) - \nu(\tilde{G}_i)).
\end{equation}
$\mu(G_i, \tilde{G}_i)$
can be positive, negative or zero.
Because $\nu(G_i)$ represents the faithfulness,
the absolute $|\mu(G_i, \tilde{G}_i)|$
measures the change in the class distribution of $v_i$ approximated by the fact $G_i$ and the counterfactual $\tilde{G}_i$.
When $|\mu(G_i, \tilde{G}_i)|$ is large, the portion $\Delta(G_i,\tilde{G}_i)$ removed from $G_i$ is likely to be the cause of $\mathbf{y}_i^{\prime}$~\cite{guo2020survey}.
The normalizer $|\Delta(G_i, \tilde{G}_i)|$ makes sure that the same difference $\nu(G_i) - \nu(\tilde{G}_i)$ caused by a small $\Delta(G_i,\tilde{G}_i)$ will be more desirable than that caused by a larger $\Delta(G_i,\tilde{G}_i)$.
It also prohibits extreme counterfactuals that remove all nodes except the target $v_i$. 
These quantities are demonstrated in Fig.~\ref{fig:desiderata}.

\begin{figure}[t]
    \centering
    \includegraphics[width=0.4\textwidth]{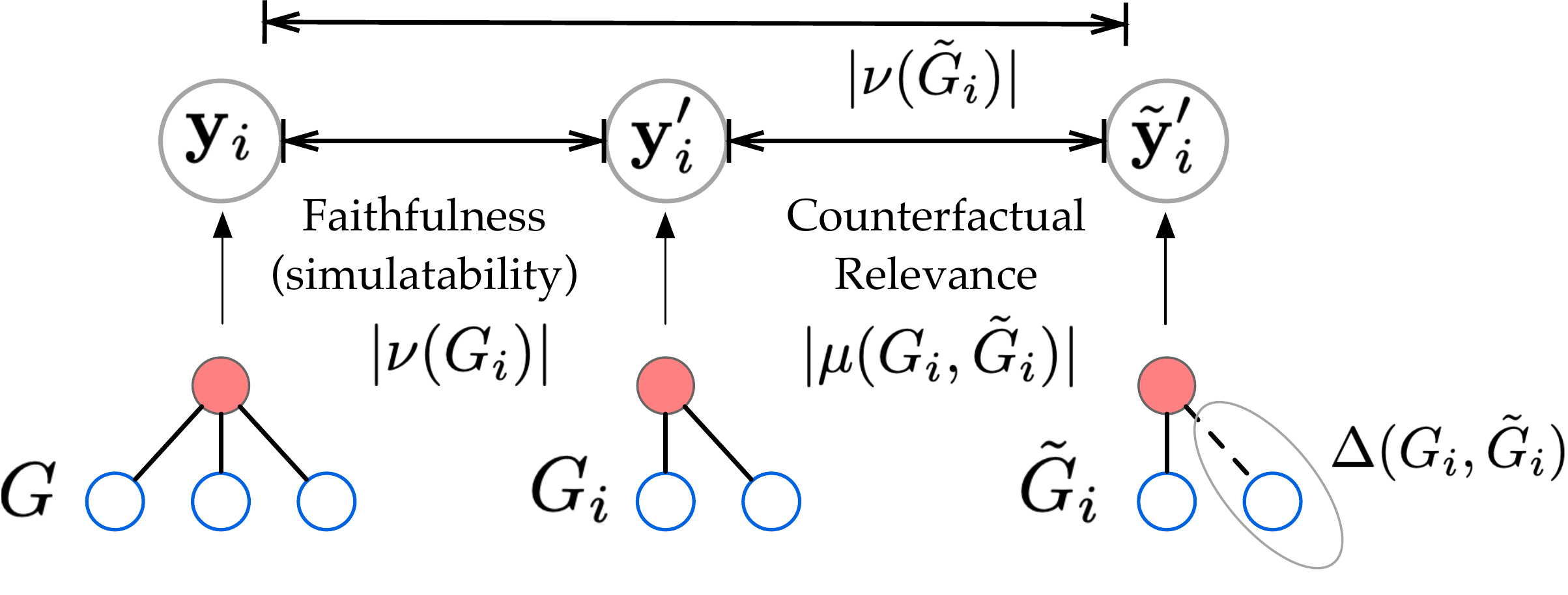}
    \caption{\small
    Two explanation metrics.
    $\mathbf{y}_i$ is the GNN prediction of $v_i$ on the full graph $G$,
    $\mathbf{y}_i^\prime$ is the GNN prediction on $G_i$, 
    and $\tilde{\mathbf{y}}_i^\prime$ on $\tilde{G}_i$. 
    Faithfulness is measured by Eq. (\ref{eq:sim}).
    The smaller the $|\nu(G_i)|$, the more faithful.
    $\Delta(G_i,\tilde{G}_i)$ is circled by the dashed line.
    The larger the $|\mu(G_i,\tilde{G}_i)|$ (Eq. (\ref{eq:cfr})), the more counterfactual relevance.
    }
    \label{fig:desiderata}
\end{figure}

\section{How Humans Perceive Explanations}

\epigraph{``System 1 operates automatically and quickly ...\\System 2 allocates attention to the effortful mental activities ...''}{\textit{Daniel Kahneman, Nobel laureate}}

We conducted a human subject study to find the roles of the two metrics in the human perception of explanations.
The two modes of thinking, System 1 and System 2, are extensively studied in psychology, as quoted above.
We conjecture that forward simulations help humans quickly screen an explanation using System 1, while reasoning using the counterfactual is a more deliberate process that requires System 2, so that humans will conduct counterfactual reasoning only after the explanation has passed System 1 screening.
Simulatability and counterfactual relevance measure how well an explanation and an associated counterfactual are received by the two Systems.

According to Fig.~\ref{fig:two_dim},
on the Cora dataset, we sample five target nodes and for each node we generate subgraphs with low and high simulatability.
This leads to ten explaining subgraphs for each subject to evaluate the simulatability.
For each explaining subgraph $G_i$, we further generate two counterfactuals $\tilde{G}_i$ that are subgraphs of $G_i$, with different counterfactual relevance.
\begin{figure}
\centering
	\includegraphics[width=0.5\textwidth]{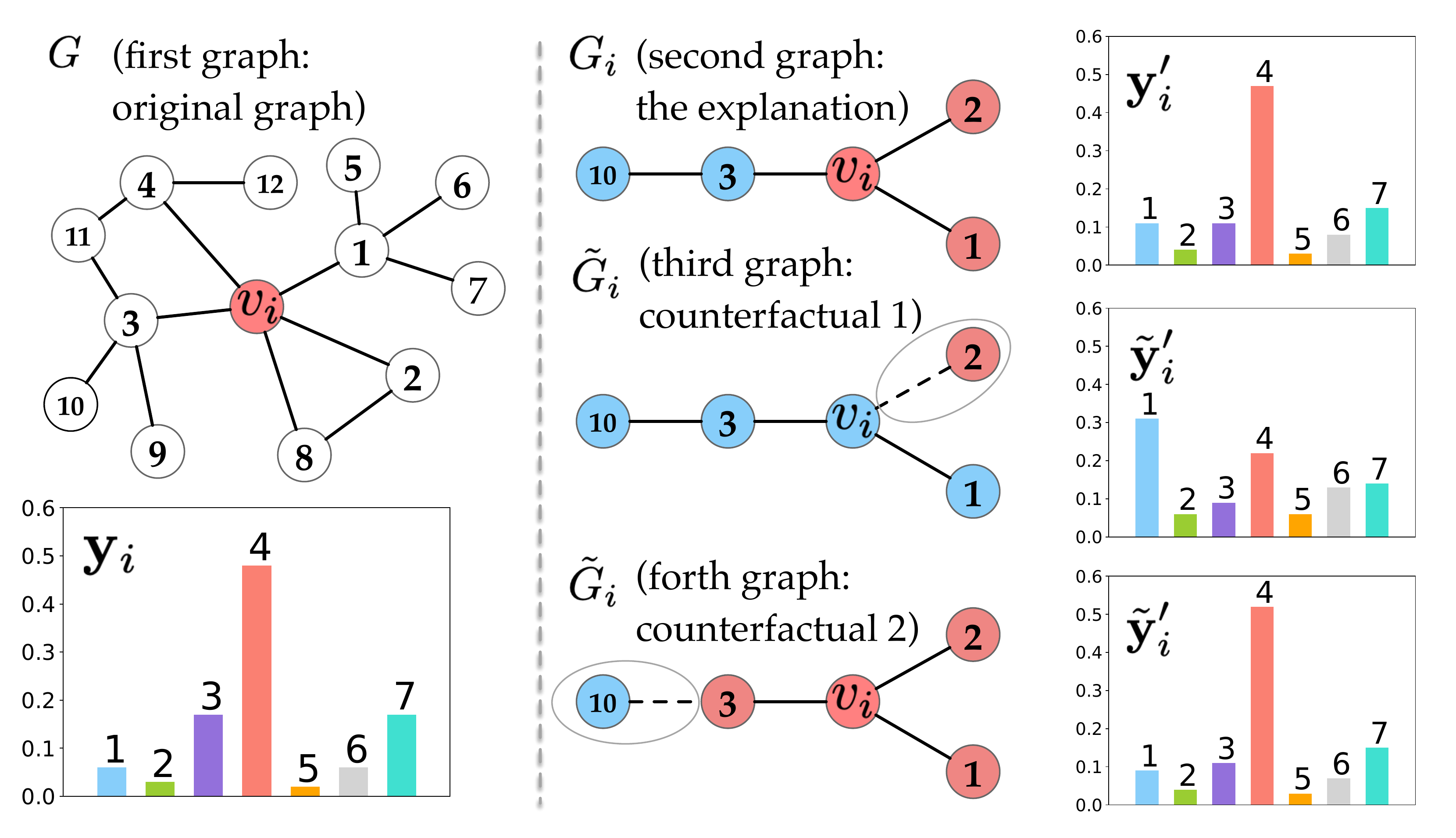}
		\caption{\small
		Sample explanations in the human study.
		Each node is a paper on Cora.
		\textit{Left:} the large graph containing $v_i$, whose prediction is to be explained.
		Predicted distributions over 7 classes are shown in histograms.
		\textit{Right:} subgraphs explaining the prediction of $v_i$, along with the class distributions predicted on the individual subgraphs (top/middle: explanation/counterfactual found by GNN-MOExp, bottom: a counterfactual with a small counterfactual relevance. Counterfactuals are constructed by removing the dashed edges).
		}
\label{fig:example}
\vspace{-.1in}
\end{figure}
Fig.~\ref{fig:example} shows one sample test case.
For each of the five nodes,
a subject will see the original graph $G$ where GNN produced the prediction $\mathbf{y}_i$,
the explanation $G_i$ that produced $\mathbf{y}_i^\prime$,
and two counterfactuals that generate two $\tilde{\mathbf{y}}_i^\prime$.
The full graph is considered to be too complicated for interpretation,
while $G_i$ is more intelligible.
The two counterfactuals allow a subject to evaluate if a removed part $\Delta(G_i,\tilde{G}_i)$ is a plausible cause of the prediction $\mathbf{y}_i^\prime$.
For each graph, we color the nodes based on the GNN's prediction, so that a subject can relate a prediction to the neighbors.
We show the predicted class distributions in histograms,
so that the predictions across the (sub)graphs can be compared conveniently.
The subjects were not told about the two metrics of the explanations but needed to understand, analyze, and then rate the explanations.

\begin{figure*}[t]
    \centering
    \begin{subfigure}[t]{0.32\textwidth}
    \includegraphics[width=\textwidth]{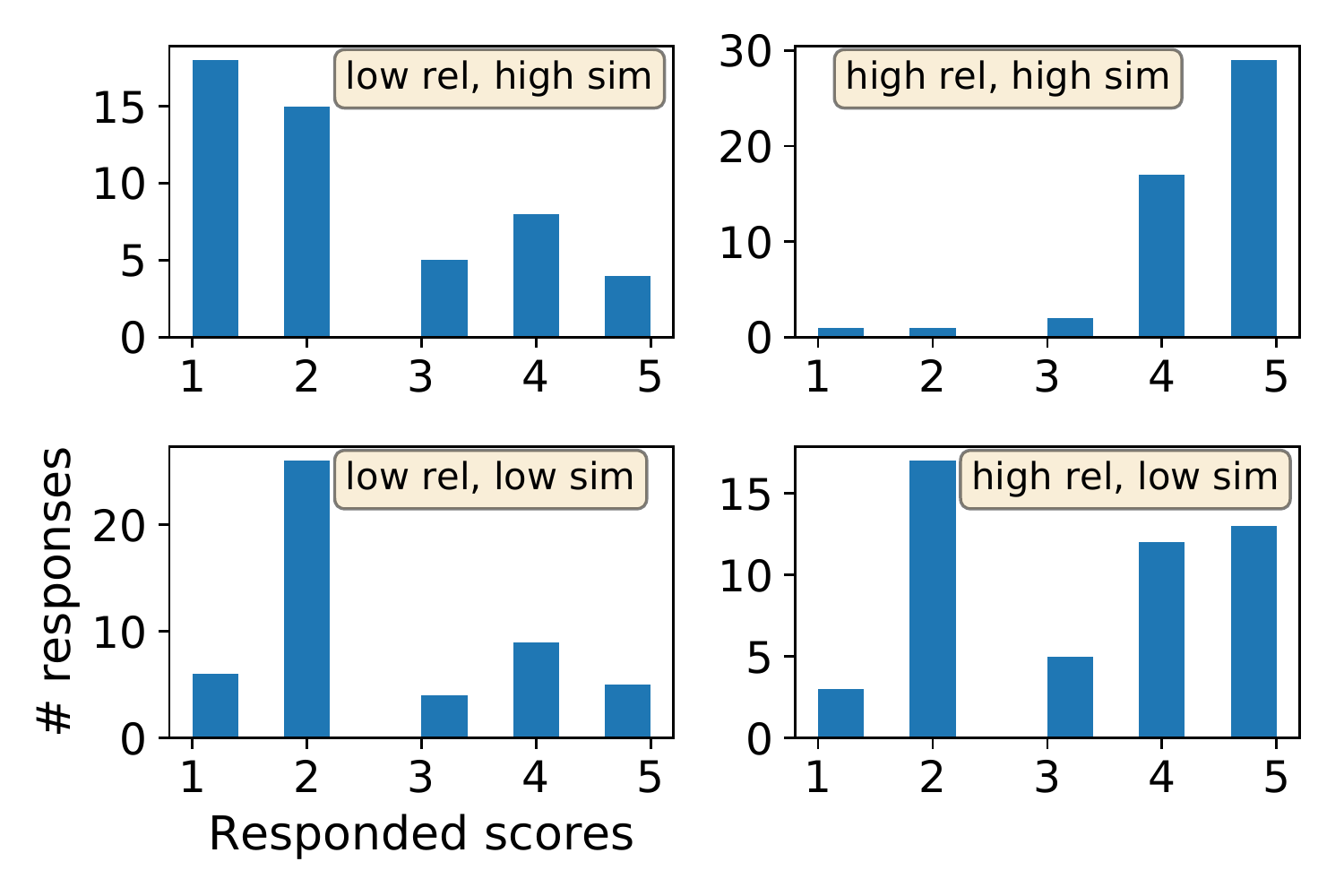}
    \caption{Histograms of responses $r_d$ in 5-point Likert scale
    under four conditions, represented in the quadrants as those in Fig.~\ref{fig:two_dim}.
    The top right quadrant has the highest acceptance.
    }
    \label{fig:interactions}
    \end{subfigure}
    \hfill
    \begin{subfigure}[t]{0.32\textwidth}
    \includegraphics[width=\textwidth]{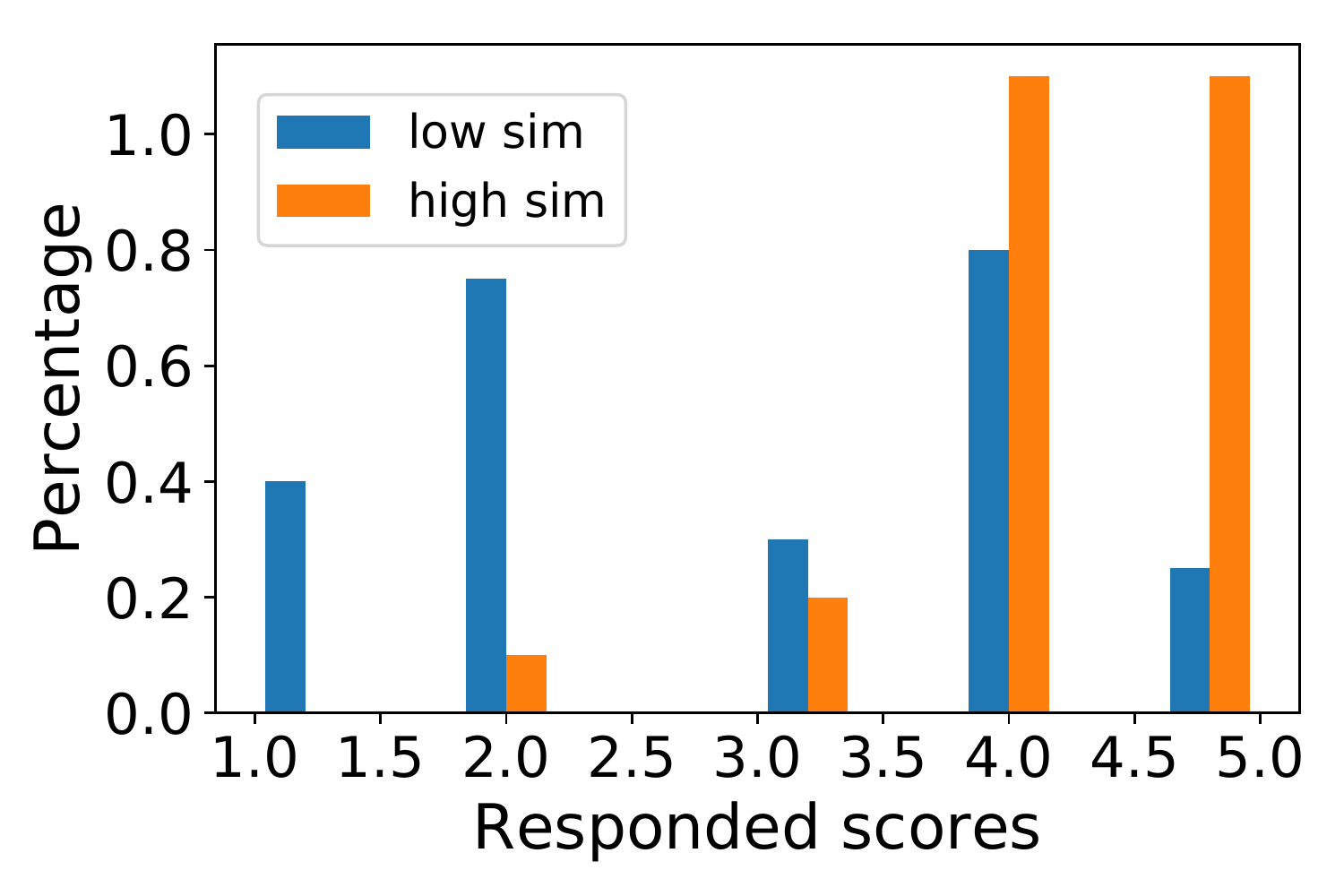}
    \caption{Regardless of counterfactual relevance, a higher (low) simulatability leads to a higher (lower) acceptance rate.
    A low simulatability can lead to low acceptance,
    though does not prohibit 4-5 points responses.}
    \label{fig:sim_acceptance}
    \end{subfigure}
    \hfill
     \begin{subfigure}[t]{0.32\textwidth}
    \includegraphics[width=\textwidth]{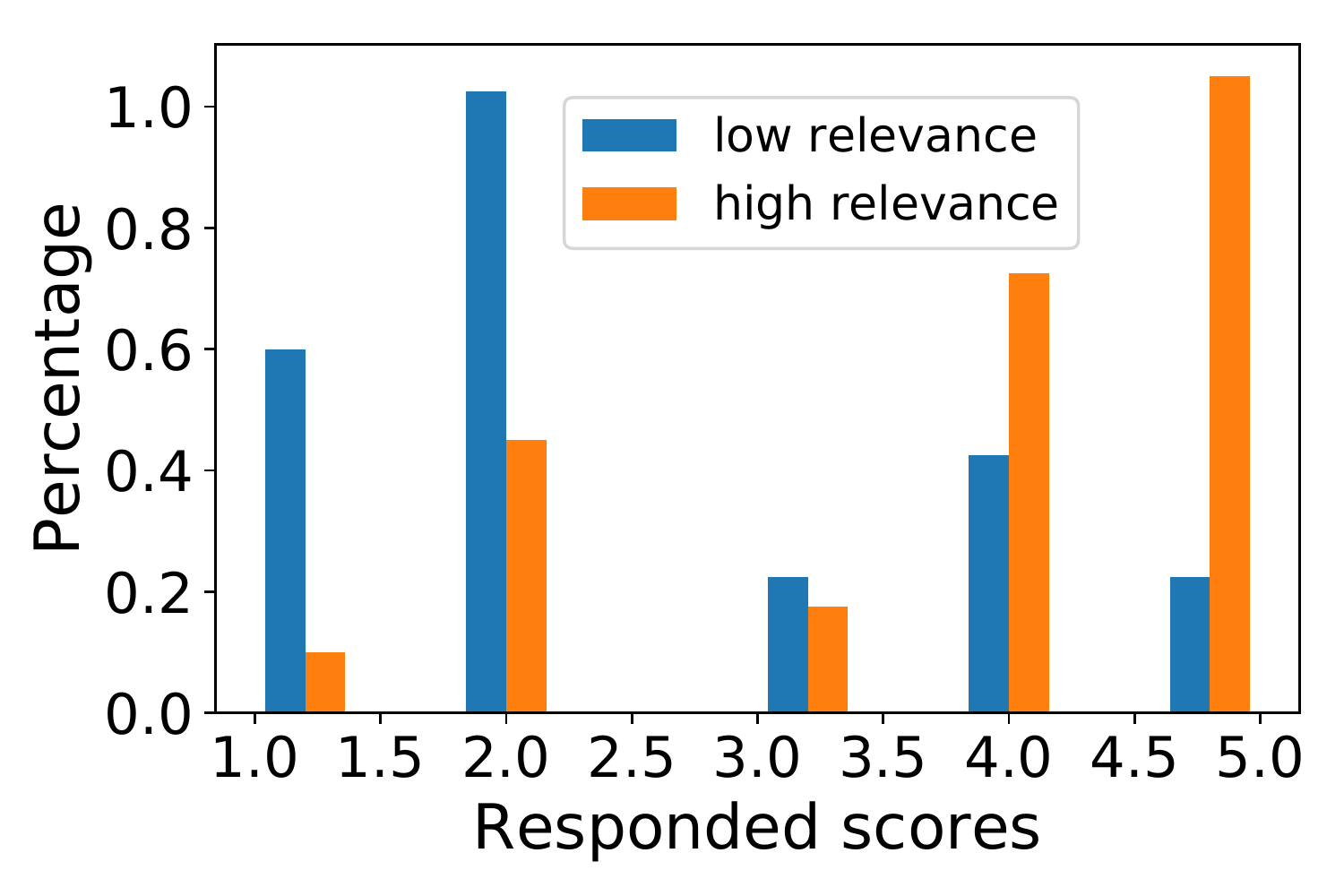}
    \caption{Regardless of simulatability,
    a high counterfactual relevance makes the found cause (the portion $\Delta$ removed from an explanation $G_i$) more convincing to human subjects.}
    \label{fig:cause_acceptance}
    \end{subfigure}
    \caption{User study.
    Fig.~\ref{fig:interactions} and ~\ref{fig:sim_acceptance} show that
    simulatability is necessary but not sufficient for explanation acceptance.}
    \label{fig:human_subject}
\end{figure*}

To avoid bias,
we frame the survey as an evaluation of a graph-based search engine and recruited subjects with search experience using Google Scholar.
The authors of this paper are excluded.
The two counterfactuals are randomly ordered.
Each subject is further trained on two additional sample cases.
During the test phase,
we ask subjects the following questions after each test case
and collect feedback ($r_\ast$ in the parentheses) in a 5-point Likert scale \textit{(1-very little (won't accept),2-little,3-not sure,4-a little, 5-very well)}:


\renewcommand{\labelenumi}{\alph{enumi}.}

\begin{enumerate}[leftmargin=*,topsep=0pt]
    \item {Simulatability} ($r_a$): How well do you think the \textbf{second} subgraph is reproducing the prediction computed in the \textbf{first} graph?
    \item {Counterfactual-1} ($r_b$): How much do you think the removed component in the \textbf{third} subgraph is an important factor leading to the histogram for the second subgraph, had it not been removed?
    \item {Counterfactual-2} ($r_c$): Same as above but replace the  the \textbf{third} subgraph with the \textbf{forth} subgraph.
    \item {Explanation acceptance} ($r_d$): How much will you accept the probabilities, if they were computed on the second subgraph rather than the first?
\end{enumerate}

\subsection{Analysis of human feedback}
The questions quantitatively reveal the human perception of the two explanation metrics.
Let the responses to the questions a, b, c, and d be $r_a$, $r_b$, $r_c$, and $r_d$, respectively.
$r_a$ measures the subject's perceived simulatability of the explanation $G_i$.
The difference between $r_b$ and $r_c$ measures the preference of a subject between two alternative counterfactuals $\tilde{G}_i$.
$r_d$ measures the subject's overall acceptance of $G_i$ as an explanation based on its simulatability and the plausibility of the causes found using the counterfactuals. 
After filtering out an obvious outlier (the responses to all questions are the same),
we have 10 subjects' responses to 10 test cases, leading to 100 scores for each of the four questions.
We draw the following conclusions based on statistical analyses.

\noindent\textbf{High simulatability helps acceptance that can be boosted by high counterfactual relevance.}
Using responses $r_d$,
a two-way analysis of variance (ANOVA)
shows that the two metrics interact strongly 
($p$-value $<0.00001$).
Fig.~\ref{fig:interactions} and~\ref{fig:sim_acceptance} confirm that a high simulatability is a prerequisite of explanation acceptance, with high counterfactual relevance being the second condition.
A low simulatability leads to more mixed acceptance,
regardless of counterfactual relevance.
There are some numbers of acceptance with low simulatability,
due to the subjects' in-depth analysis of the cases that leads to a final acceptance.

\noindent\textbf{Simulatability can predict acceptance of explanations.}
We conducted a $t$-test on the responses $r_d$ from two groups: one has cases with low simulatability and the other has cases with high simulatability.
The $p$-value is almost zero, indicating that the degree of acceptance differs significantly between the groups.
The $t$-statistic is $-6.4$.
After taking into account the within-group variances and the sample size,
we conclude that the acceptance of a less simulatable explanation is less than that of a more simulatable explanation.
Fig.~\ref{fig:sim_acceptance} further confirm this conclusion.

\noindent\textbf{A higher counterfactual relevance makes a reason more likely perceived as ``the cause''.}
While there can be several factors that jointly lead to the GNN prediction $\tilde{\mathbf{y}}_i$, humans tend to accept the one with high counterfactual relevance as ``the cause'',
compared to those with low counterfactual relevance.
We conducted a $t$-test between the responses $r_b$ and $r_c$.
The tests show that a higher counterfactual relevance is more convincing (all $p$-value $<0.01$), regardless of simulatability (see Fig.~\ref{fig:cause_acceptance}).
However, when simulatability is low,
the presented ``cause'' is less convincing (see bottom two subfigures of Fig.~\ref{fig:interactions}).
\textbf{Caution}: ``the cause'' presented by a counterfactual may not be the only or the true cause of the prediction $\tilde{\mathbf{y}}_i$, due to confounders. See Section~\ref{sec:confounder}.

\section{Multi-objective explanations of GNN}

\begin{figure*}
\begin{center}
\includegraphics[width=\textwidth]{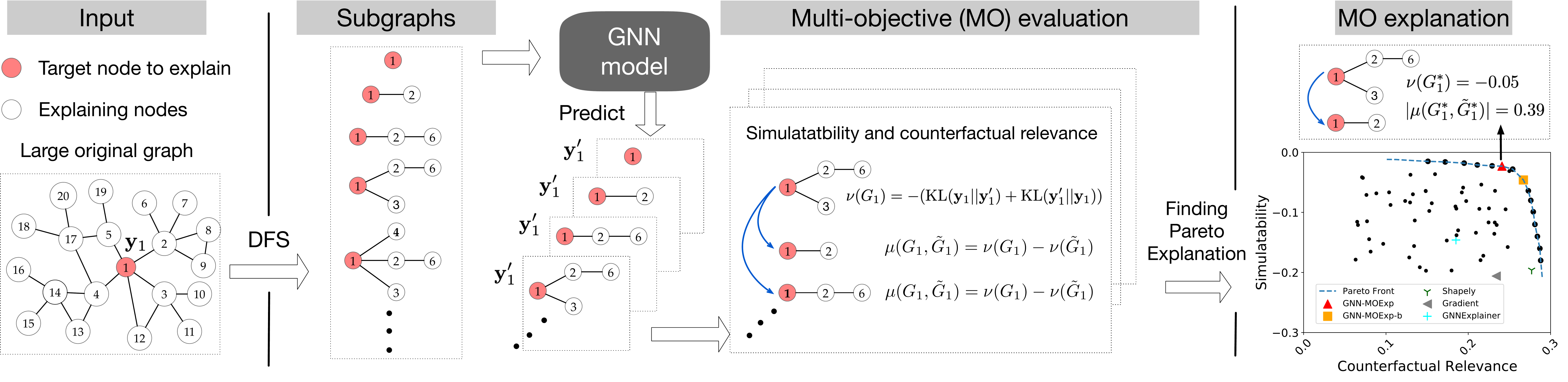}
\caption{\small The workflow of finding Pareto optimal GNN explanations with high simulatability and counterfactual relevance.
The subgraphs are enumerated by DFS and the two objectives are computed on the enumerated explanations and their counterfactuals $(G_i, \tilde{G}_i)$ for each node $v_i$ to be explained.
Pareto optimal explanation $(G_i^\ast, \tilde{G}_i^\ast)$ that are high (but not necessarily the highest) in both metrics are selected.}
\label{GNNCFE}
\end{center}
\vspace{-0.4cm}
\end{figure*}

Given the human study results,
we aim to solve the following multi-objective optimization problem.
\begin{defn}
Given a graph $G$ and a GNN model $\boldsymbol{\theta}$,
on any target node $v_i \in V$,
extract an explanation subgraph $G_i\subset G$ and a counterfactual subgraph $\tilde{G}_i$, 
where
$v_i \in \tilde{G}_i\subset G_i$, $G_i$ contains no more than $C$ nodes and is acyclic, so that $\nu(G_i)$ and $|\mu(G_i,\tilde{G}_i)|$ are maximized:
\begin{equation}
\label{eq:opt}
\arraycolsep=1.4pt\def\arraystretch{1.5}
    \begin{array}{cc}
        \hspace{.2in}
         \max\limits_{G_i,\tilde{G}_i} & F(G_i,\tilde{G}_i)=(\nu(G_i),  |\mu(G_i,\tilde{G}_i)|) \\
         \textnormal{s.t.} & v_i \in \tilde{G}_i\subset G_i \subset G, 
         \hspace{.05in}
         |G_i|\leq C,
         \hspace{.05in}
         G_i\textnormal{ acyclic},
    \end{array}
\end{equation}
\end{defn}
For simplicity of the explanation, we restrict $G_i$ to contain no more than $C$ nodes~\cite{Miller1956}.
The limit to $C$ nodes also reduces the degree, coreness, and centrality of any nodes in $G_i$, and improves human reaction time when reasoning with $G_i$~\cite{Lynn29407}. 
We restrict the explanations to be acyclic graphs~\cite{Vu2020PGMExplainerPG},
since a cycle can
lead to self-proof and explanations such as ``\textit{
Alice is a database researcher because she cited a paper of Bob,
who is a database researcher since he cited Alice's paper''.
}

The optimization is bi-objective and the objective vector function $F$ has two scalar objectives.
We don't use a single scalar objective function, such as 
$\nu(G_i)+\lambda |\mu(G_i,\tilde{G}_i)|$,  not only because that $\lambda$ can be hard to specify, but also that trading one objective for the other is not desirable according to the human subject study (either low simulatability or counterfactual relevance suppresses human acceptance of the explanation and the counterfactual).
Beyond being multi-objective, the solution space of all possible $G_i$, defined by the constraints in the above optimization problem, is exponentially large and discrete and no polynomial-time algorithm is known to search the space.
The gradient-based methods in~\cite{ying2019gnn,pope2019explainability} and the search-based methods
in
~\cite{Russell2019,Ustun2019,Mothilal2020fat,yuan2020xgnn,Vu2020PGMExplainerPG} can only maximize one of the objective functions and do not guarantee Pareto optimality, i.e., efficient trade-off between objectives.
We follow the search-based explanation generation paradigm,
but aim at finding the Pareto front and selecting one particular Pareto efficient explanation with well-balanced objectives.


\subsection{Search for Pareto optimal explanations}
\label{sec:algorithm}
The algorithm, GNN-MOExp (Graph Neural Network Multi-Objective Explanations) is shown in Fig.~\ref{GNNCFE}.
We first apply a depth-first search (DFS) to explore the space of subgraphs $G_i$ for $v_i$. 
Since the prediction $\mathbf{y}_i$ of the target $v_i$ does not depend on nodes that are more than $L$ hops away from $v_i$,
the search is restricted to the dependent neighbors.
A canonical ordering of the edges is determined by a breadth-first search (BFS) before running the DFS, ensuring no subgraph will be enumerated more than once.
The BFS also canonically numbers the nodes to avoid isomorphism test during graph lookup: the same graph will be represented by a unique array of edges with canonical node numbering.
Starting from the subgraph containing only $v_i$,
the DFS expands the subgraph by adding an un-visited edge adjacent to the current subgraph.
The constraints in Eq. (\ref{eq:opt}) are used in pruning the search space.
After all valid candidate subgraphs containing the edge have been explored,
the edge is flagged and
will not be visited in future.
The enumeration will be completed when all edges within the neighborhood are processed.

The GNN model has to be run on each enumerated subgraph $G_i$
and the two metrics $\nu(G_i)$ and $\mu(G_i,\tilde{G}_i)$ are computed
by Eq. (\ref{eq:sim}) and Eq. (\ref{eq:cfr}).
Since $G_i$ contains at most $C$ nodes, the cost is low.
To avoid repetitive calculation of $\nu(G_i)$ when calculating $\mu(G_i,\tilde{G}_i)$,
a hash table is used to record $\nu(G_i)$ for each subgraph.
$G_i$ becomes a counterfactual of all subgraphs that are the descents of $G_i$ in the DFS search tree.

After evaluating each subgraph and its counterfactuals,
we need to find the optimal explanation so that both metrics are high.
However, the two metrics can be competing and it is hard to find an explanation that outperforms all others in \textit{both} metrics.
We aim to find Pareto optimal (efficient) explanations, that are optimal in the sense that it cannot be outperformed by another explanation in \textit{both} metrics~\cite{miettinen1998nonlinear}.
We need the following definitions.

\begin{defn}{(Pareto dominance)} Let $F_1(G)=\nu(G)$ and $F_2(G)=|\mu(G,\tilde{G})|$. $\tilde{G}_1\subset G_1\subset G$, $\tilde{G}_2\subset G_2 \subset G$.
If $(G_1,\tilde{G}_1)$ Pareto dominates $(G_2,\tilde{G}_2)$, then $\forall i \in \lbrace1,2\rbrace, F_i(G_2,\tilde{G}_2) \le F_i(G_1,\tilde{G}_1) \wedge \exists i \in \lbrace1,2\rbrace, F_i(G_2,\tilde{G}_2) < F_i(G_1,\tilde{G}_1)$, denoted as $(G_2,\tilde{G}_2) \prec (G_1,\tilde{G}_1)$.
\end{defn}
\begin{defn}{(Pareto optimality)}.
$(G_1,\tilde{G}_1)$ is Pareto optimal
if and only if $\nexists (G_2,\tilde{G}_2) \prec (G_1,\tilde{G}_1)$.
\end{defn}
\begin{defn}{(Pareto optima)}
The set of all Pareto optimal solutions: $P_s:=\lbrace (G_1,\tilde{G}_1) | \nexists (G_2,\tilde{G}_2), (G_1,\tilde{G}_1) \prec (G_2,\tilde{G}_2)  \rbrace $.
\end{defn}
\begin{defn}{(Pareto optimal front)}.
The set consists of the function values of the Pareto optimal set:
$P_{F}:=\lbrace F(G_i,\tilde{G}_i) \mid (G_i,\tilde{G}_i) \in P_{s}\rbrace$.
\end{defn}

However, 
explanations $(G_i, \tilde{G}_i)$
on the Pareto front can be low in one objective while being high in another, and is thus not useful.
We design a simple method to find Pareto optimal explanations that are: 1) dominating other explanations, and 2) likely simultaneously optimal in individual metrics (without guarantee).
In particular, we sort the explanations and their counterfactuals $(G_i, \tilde{G}_i)$ along the simulatability and counterfactual relevance, independently.
Let the ranking position of $(G_i, \tilde{G}_i)$ in the two rankings be denoted by $r_1(G_i, \tilde{G}_i)$ and $r_2(G_i, \tilde{G}_i)$ (the smaller the better).
We define the comprehensive ranking $R(G_i, \tilde{G}_i)$ be
\begin{equation}
\label{eq:rank}
R(G_i, \tilde{G}_i) = r_1(G_i, \tilde{G}_i) + r_2(G_i, \tilde{G}_i).
\end{equation}
Finally, we select the $(G_i, \tilde{G}_i)$ with the best comprehensive ranking, denoted by $(G_i^\ast, \tilde{G}_i^\ast)$ as the final explanation.

One possible baseline is to use the so-called preference vector to select a Pareto optimal solution that satisfies some weighted balance between the objectives~\cite{mahapatra20a}.
We found this method hard to use in our case:
the two objectives are of different ranges,
which vary across different target nodes.
In contrast, the ranking-based approach handles the heterogeneity.
We did not present this baseline since it significantly underperforms our method.
A more competitive baseline is to find $(G_i, \tilde{G}_i)$ whose rankings in the two objectives are well balanced. 
We compare our approach with this baseline in the experiments.
Since the Pareto front is non-convex and contains dents that have well-balanced but low objective values, the above baseline may not work well.


The explanation chosen by the comprehensive ranking is in the Pareto front, as shown by the following theorem.
\begin{thm}
\label{thm:ranking_Pareto}
The ranking-based method finds a solution $(G_i^\ast, \tilde{G}_i^\ast)$ that's on the Pareto front.
\end{thm}
\begin{proof}
If $(G_i^\ast, \tilde{G}_i^\ast)$ is not a Pareto optimal solution, then there is $(G_i, \tilde{G}_i)$ that dominates $(G_i^\ast, \tilde{G}_i^\ast)$.
By definition, $(G_i, \tilde{G}_i)$ must be ranked higher than $(G_i^\ast, \tilde{G}_i^\ast)$ in at least one objective, while in the other objective the two are at least equal.
According to the definition of comprehensive ranking,
$R(G_i, \tilde{G}_i)<R(G_i^\ast, \tilde{G}_i^\ast)$ and $(G_i, \tilde{G}_i)$ would have been chosen by the explanation selection algorithm.
\end{proof}

\noindent\textbf{Complexity of the Algorithm.}
Regarding the DFS,
in the best case,
$v_i$ is on one end of a linear chain and the time complexity is $O(1)$.
In the worst case, 
the number of subgraphs of a complete graph with $n$ nodes is exponential,
and the complexity is $O(e^n)$.
Many real-world graphs are sparse and the complexity is more likely to be polynomial. 
The depth of GNN $L$ is usually limited ($\leq 3$) due to the over-smoothing effect of aggregation~\cite{Li2018DeeperII} and the number of nodes searched depends on the size of the $L$-hop neighborhood of the target node.
We show in Fig.~\ref{fig:running_time} that the running time of the subgraph search is practically low.

It seems that one has to find the Pareto front and then use the comprehensive ranking to find the best explanation.
To eliminate all dominated solutions, the time complexity is quadratic in the number of enumerated subgraphs.
However, Theorem~\ref{thm:ranking_Pareto} says that the comprehensive ranking already points to a solution on the Pareto front and the overall time complexity is just linear in the number of enumerated subgraphs, using the heap data structure. 

\begin{figure}[h]
\centering
    \includegraphics[width=0.34\textwidth]{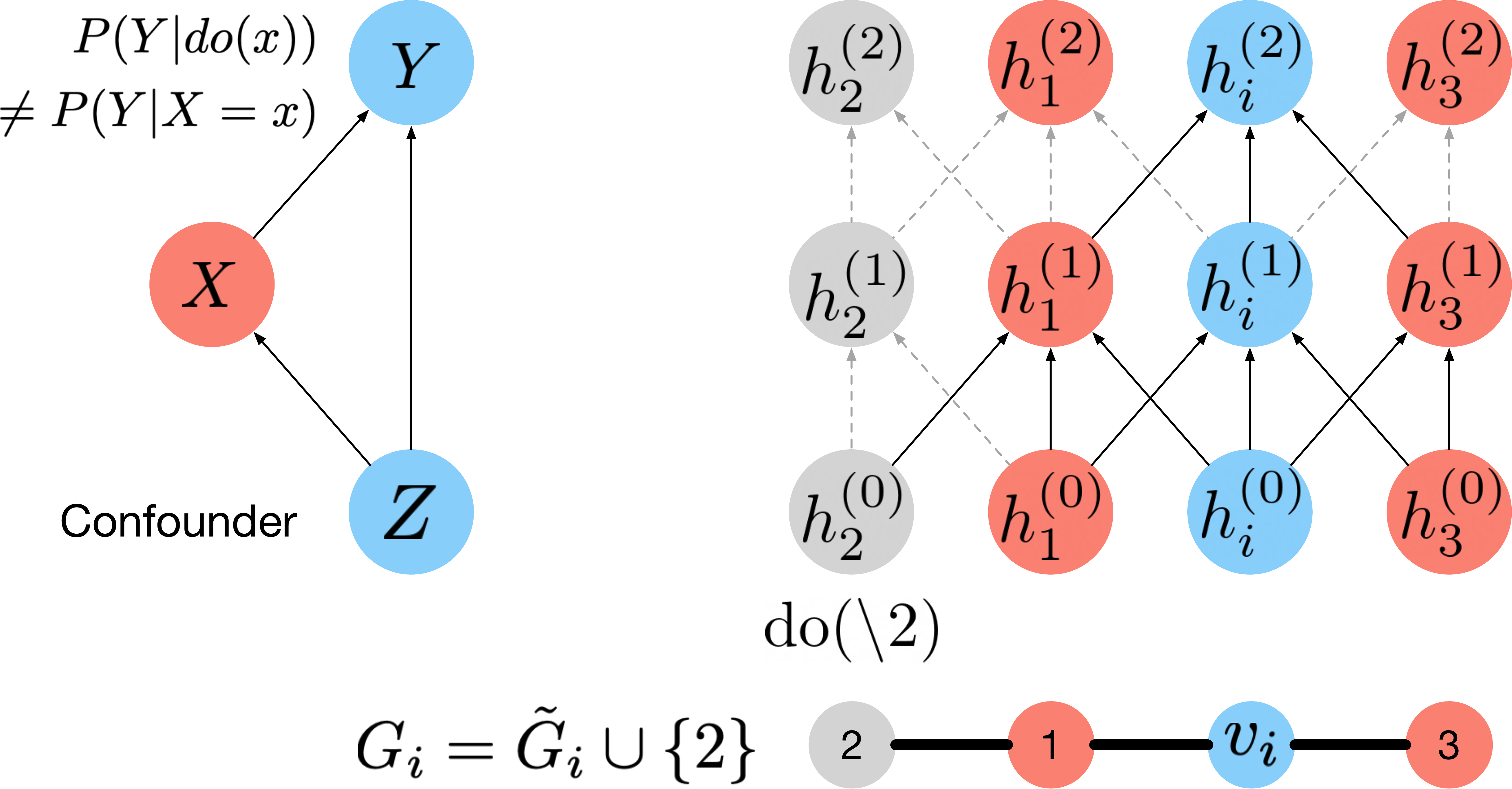}%
    \caption{\small
    Confounders in GNN.
    \textit{Left}: a structural causal model with $F$ being the common cause of both $X$ and $Y$, and $Z$ is a confounder that makes $P(Y|do(x))\neq P(Y|X=x)$.
    \textit{Right}: at the bottom, $G_i$ explains the prediction $h_i^{(2)}$ on node $v_i$ and node $2$ can be removed from $G_i$ as an intervention to obtain a counterfactual explanation  $\tilde{G}_i$. 
    Above $G_i$ is a computation graph that represents the rollout of the structural equations (\ref{eq:message})-(\ref{eq:hidden}).
    Arrows are dependencies among the nodes on the computational graph and dashed lines are not relevant to $h_i^{(2)}$.
    The variable $h_i^{(0)}$ is a common cause of $h_1^{(1)}$ and $h_i^{(2)}$ and therefore confounds the effect of $h_2^{(0)}$ on $h_i^{(2)}$ through $h_1^{(1)}$.
    There are other confounders, and
    the effect of the intervention on $h_i^{(2)}$ should be adjusted for all confounders. 
    }
    \label{fig:gi_example}
\end{figure}
\subsection{Confounders}
\label{sec:confounder}
\textit{Confounders} are variables that impact both causes and outcome \cite{Pearl2009}.
Fig.~\ref{fig:gi_example} shows the concepts of confounder that leads to the Back-Door adjustment:
\begin{equation}
    P(Y|do(x))=\sum_{z}P(Y|X=x,Z=z) P(Z=z),
\end{equation}
which is in general not the same as $P(Y|X=x)$.
For $G_i$ in the figure,
the counterfactual explanation $\tilde{G}_i$ is obtained by the intervention of removing $\Delta=\{2\}$ from $G_i$.
Humans may think that $\Delta$ is ``the cause'' of the output $h_i^{2}$.
However, this is not true due to confounders,
as shown in Fig.~\ref{fig:gi_example}.

\subsection{Connection to Shapley values}
\label{sec:connection}
There is a close relationship between counterfactual explanations and Shapley values ~\cite{shapley1953value,chen2018shapley}.
As an explanation,
Shapley values are the importance of the factors contributing to the predictions to be explained.
One can consider the portion $\Delta$ removed from a subgraph $G_i$ as a contributor, and by averaging $\Delta$'s contributions over all possible $G_i$ that contain $\Delta$ (denoted by $\mathcal{ S}(\Delta; G)$),
we obtain the Shapley value of $\Delta$:
\begin{equation}
\label{eq:sv}
\textnormal{SV}(\Delta) :=  \frac{1}{|\mathcal{ S}(\Delta; G)|} \sum_{G_i\in \mathcal{ S}(\Delta; G)} \mu(G_i,G_i - \Delta).
\end{equation}
The contribution $\mu(G_i,G_i+\Delta)$ follows the definition of Shapley values and can be positive, negative, or zero.
Instead, counterfactual relevance $|\mu(G_i,G_i-\Delta)|$ is always non-negative and gives the magnitude of the importance of $\Delta$.


\subsection{Robustness and sanity check of explanations}
An accurate explanation of a prediction should vary according to the underlying mechanism that generates the prediction~\cite{Adebayo2018},
and should remain the same under irrelevant perturbations~\cite{Ghorbani2017}.
\begin{figure}[h!]
\centering
    \includegraphics[width=0.45\textwidth]{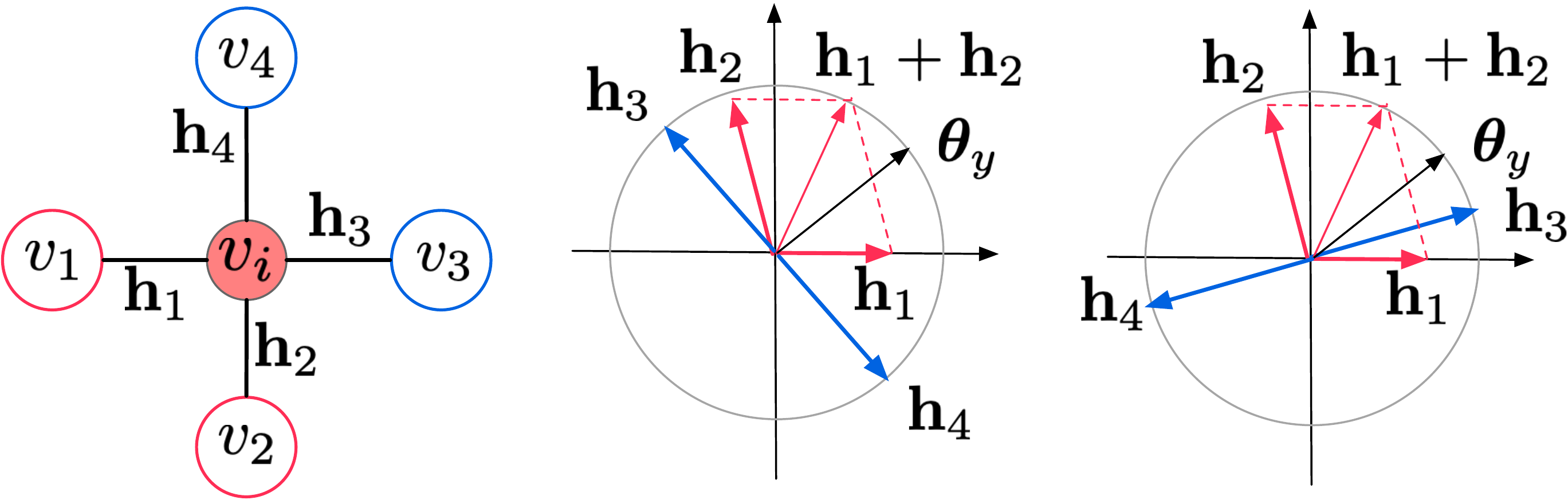}%
    \caption{\small
    Manipulate a GNN explanation.
    \textit{Left}: original graph.
    \textit{Center}: messages $\mathbf{h}_1$ and $\mathbf{h}_2$ cause the prediction on $v_i$, while $\mathbf{h}_3$ and $\mathbf{h}_4$ are irrelevant.
    \textit{Right}: $\mathbf{h}_3$ and $\mathbf{h}_4$ are rotated to perturb a gradient-based explanation, though the prediction of class $y$ remains the same.
    }
    \label{fig:sanity_check_robustness}
\end{figure}
\begin{defn}
The robustness of a subgraph explanation $G_i$
is the degree of the change in $G_i$ under perturbations that are irrelevant to the mechanism that generates $\mathbf{y}_i$.
\end{defn}
We assume a one-layer GNN ($L=1$) with parameter $\boldsymbol{\theta}\in\mathbb{R}^{K\times d}$,
where $K$ is the total number of classes to be predicted and $d$ is the number of features of the nodes.
We use the graph in Fig.~\ref{fig:sanity_check_robustness} \textit{Left} to demonstrate the difference in the robustness of explanations found by GNN-MOOExp and prior gradient-based methods. 
Gradient-based methods~\cite{pope2019explainability,ying2019gnn,baldassarre2019explainability} find explanations using
the gradient of the following faithfulness loss function with respect to a mask $M\in[0,1]^{|V|\times |V|}$ over the adjacency matrix $A\in\{0,1\}^{|V|\times |V|}$:
\begin{eqnarray}
    \ell(M;G,\boldsymbol{\theta})
    &=&-\textstyle\sum_{y=1}^K\mathbbm{1}[y_i=y]\log P(y|A\odot M;\boldsymbol{\theta}),\nonumber\\
    P(y|A\odot M;\boldsymbol{\theta})
    &=&
    \textnormal{softmax}\left(\textstyle\sum_{j=1}^n M_{ij}A_{ij}\mathbf{h}_j^\top\boldsymbol{\theta}_y \right)\nonumber,
\end{eqnarray}
where $\boldsymbol{\theta}_y$ is the $y$-th row of $\boldsymbol{\theta}$.
As we are explaining a GNN prediction,
$y_i=y$ is the predicted class and not necessarily the ground truth class of $v_i$.
$\mathbf{h}_j$ is the input feature vectors of the neighbor $v_j$ of $v_i$.
The target GNN model will set all entries of $M$ to 1 so that all neighbors of $v_i$ are retained.
In Fig.~\ref{fig:sanity_check_robustness} \textit{center},
the neighbors' features satisfy $\mathbf{h}_3=-\mathbf{h}_4$ so that the relevant neighbors to $\mathbf{y}_i$ are just $v_1$ and $v_2$, with representations $\mathbf{h}_1$ and $\mathbf{h}_2$,
whose sum is closer to $\boldsymbol{\theta}_y$ than to $\boldsymbol{\theta}_{y'}$ for any $y'\neq y$.
The gradient of $\ell$ w.r.t. $M_{ij}$ is
\begin{equation}
    (\mathbbm{1}[y_i=y] - P(y)) A_{ij}\boldsymbol{\theta}_y^\top \mathbf{h}_j.
\end{equation}
The importance of the edge $(i,j)$ is the magnitude of the above gradient, essentially determined by the correlation between $\boldsymbol{\theta}_y$ and $\mathbf{h}_j$.
In Figure~\ref{fig:sanity_check_robustness} \textit{center},
since both $\mathbf{h}_3$ and $\mathbf{h}_4$ are orthogonal to $\boldsymbol{\theta}_y$, gradient-based methods will never have $v_3$ and $v_4$ in their explanations.
When $\mathbf{h}_3$ and $\mathbf{h}_4$ are rotated so that $\mathbf{h}_3$ is more similar to $\boldsymbol{\theta}_{y}$ than $\mathbf{h}_2$ while $\mathbf{h}_3=-\mathbf{h}_4$ remains, the gradient-based explanation will include $v_3$, even the prediction remains the same.
The rotations 
are irrelevant to how $\mathbf{h}_1+\mathbf{h}_2$ leads to the prediction $y_i$.
On the other hand,
$\boldsymbol{\theta}_y$ is closer to $\mathbf{h}_1+\mathbf{h}_2$ than to $\mathbf{h}_1+\mathbf{h}_3$ or $\mathbf{h}_2+\mathbf{h}_3$ if only subgraphs with three nodes ($C=3$) are allowed.
As a result, GNN-MOExp still finds the same optimal subgraph containing $v_i$, $v_1,$ and $v_2$, even after the rotations and is thus more robust.

Another aspect is that an explanation should faithfully reflect how a \textit{changing} $\mathbf{y}_i$ is generated and is different from simulatability that focuses on explaining a \textit{static} mechanism that generates a fixed $\mathbf{y}_i$. Formally,
\begin{defn}
A sanity check of an explanation $G_i$ of a GNN model's prediction $\mathbf{y}_i$ verifies if $G_i$ changes when the mechanism that generates $\mathbf{y}_i$ changes.
\end{defn}
A sanity check is a necessary (but not a sufficient) condition for an explanation to be a faithful surrogate of the full model:
not passing the sanity check indicates that an explanation is not reflecting the input-output relationship encoded by the GNN.
When debugging a GNN model to identify whether the model or the graph data are manipulated or polluted,
passing the sanity check means the explanations can reveal the malicious attacks to the model or data.
The prior work~\cite{Adebayo2018} proposed a sanity check for deep neural networks on images and does not address sanity checks for GNN on graphs.
We conduct sanity checks for GNN-MOExp in Section~\ref{sec:exp_sanity_check_robustness}.

\section{Experiments}
\label{sec:experiments}


\begin{table}[h]
    \scriptsize
    \caption{Nine networks from four application domains.}
    \centering
    \begin{tabular}{c|c|c|c|c|c}
    \toprule
    \textbf{Datasets} & \textbf{Classes} & \textbf{Nodes} & \textbf{Edges} & \textbf{Edge/Node} & \textbf{Features}\\
    \midrule
    \textbf{Cora} & 7 & 2,708 & 10,556 & 3.90 & 1,433\\
    \textbf{Citeseer} & 6 & 3,321 & 9,196 & 2.78 & 3,703\\
    \textbf{PubMed} & 3 & 1,9717 & 44,324 & 2.24 & 500\\
    \midrule
    \textbf{Musae-F} & 4 & 2,2470 & 342,004 & 15.22 & 4,714\\
    \textbf{Musae--G} & 2 & 37,700 & 578,006 & 15.33 & 4,005\\
    \midrule
    \textbf{Amazon-C} & 4 & 13,752 & 574,418 & 41.77 & 767\\
    \textbf{Amazon-P} & 6 & 7,650 & 287,326 & 37.56 & 745\\
    \midrule
    \textbf{Coauthor-C} & 13 & 18,333 & 327,576 & 17.87 & 6,805\\
    \textbf{Coauthor-P} & 2 & 34,493 & 991,848 & 28.76 & 8,415\\
    \bottomrule
    \end{tabular}
     \label{tab:datasets}
\end{table}

\begin{table*}[!htb]
\caption{\small Overall performance (the higher (\textcolor{red}{$\uparrow$}) the simulatability and the counterfactual relevance, the better).
$\circ$ indicates the runner-up methods and $\bullet$ indicates the best method certified by statistically significant $t$-tests (pairwise t-test at 5\% significance level).
The worst performances are underlined and the second-worst performances are under wave lines.
}
\centering
\scriptsize
\begin{tabular}{c
                ||@{\hspace*{1mm}}c
                @{\hspace*{1mm}}c
                @{\hspace*{1mm}}c
                @{\hspace*{1mm}}c
                @{\hspace*{1mm}}c
                @{\hspace*{1mm}}c
                @{\hspace*{1mm}}c
                @{\hspace*{1mm}}c 
                @{\hspace*{1mm}}c 
                ||c
                  @{\hspace*{1mm}}c
                  @{\hspace*{1mm}}c 
                  @{\hspace*{1mm}}c 
                  @{\hspace*{1mm}}c 
                  @{\hspace*{1mm}}c 
                  @{\hspace*{1mm}}c
                  @{\hspace*{1mm}}c
                  @{\hspace*{1mm}}c 
                  @{\hspace*{1mm}}c}
\toprule
    \multirow{2}{*}{\textbf{Datasets}} &
    \multicolumn{9}{c||}{\textbf{Simulatability (\textcolor{red}{$\uparrow$})}} &
    \multicolumn{9}{c}{\textbf{Counterfactual Relevance (\textcolor{red}{$\uparrow$})}} \\
 \cline{2-19}

& 
\textbf{RND} & 
\textbf{EMB} & 
\textbf{Grad} & 
\textbf{GAT} & 
\textbf{GNNExp} & 
\textbf{PGExp} & 
\textbf{Shapley} & 
\textbf{MOEB} & 
\textbf{GNN-MOExp} & 
\textbf{RND} & 
\textbf{EMB} & 
\textbf{Grad} & 
\textbf{GAT} & 
\textbf{GNNExp} & 
\textbf{PGExp} & 
\textbf{Shapley} & 
\textbf{MOEB} & 
\textbf{GNN-MOExp} & 
\\

\hline
Cora 
& \makecell[c]{-0.196}
& \makecell[c]{-0.252}
& \makecell[c]{\underline{-0.530}}
& \makecell[c]{-0.243}
& \makecell[c]{-0.213}
& \makecell[c]{\uwave{-0.272}}
& \makecell[c]{-0.256}
& \makecell[c]{-0.108$\circ$}
& \makecell[c]{\textbf{-0.049}$\bullet$}
& \makecell[c]{0.240}
& \makecell[c]{0.260}
& \makecell[c]{0.330}
& \makecell[c]{0.243}
& \makecell[c]{\uwave{0.225}}
& \makecell[c]{\underline{0.217}}
& \makecell[c]{\textbf{0.615}$\bullet$}
& \makecell[c]{0.455}
& \makecell[c]{0.467$\circ$}
\\

Citeseer
& \makecell[c]{-0.051}
& \makecell[c]{-0.054}
& \makecell[c]{\uwave{-0.066}}
& \makecell[c]{-0.050}
& \makecell[c]{-0.056}
& \makecell[c]{-0.058}
& \makecell[c]{\underline{-0.068}}
& \makecell[c]{-0.044$\circ$}
& \makecell[c]{\textbf{-0.039}}
& \makecell[c]{0.114}
& \makecell[c]{0.116}
& \makecell[c]{0.116}
& \makecell[c]{0.115}
& \makecell[c]{\uwave{0.113}}
& \makecell[c]{\underline{0.112}}
& \makecell[c]{\textbf{0.178}}
& \makecell[c]{0.156}
& \makecell[c]{0.159$\circ$}
\\

PubMed
& \makecell[c]{-0.081}
& \makecell[c]{-0.110}
& \makecell[c]{\underline{-0.365}}
& \makecell[c]{-0.117}
& \makecell[c]{-0.086}
& \makecell[c]{-0.125}
& \makecell[c]{\uwave{-0.129}}
& \makecell[c]{-0.041$\circ$}
& \makecell[c]{\textbf{-0.010}$\bullet$}
& \makecell[c]{0.112}
& \makecell[c]{0.129}
& \makecell[c]{0.200}
& \makecell[c]{0.117}
& \makecell[c]{\uwave{0.100}}
& \makecell[c]{\underline{0.099}}
& \makecell[c]{\textbf{0.330}$\bullet$}
& \makecell[c]{0.235}
& \makecell[c]{0.248$\circ$}
\\

\midrule

Musae-F
& \makecell[c]{-0.972}
& \makecell[c]{\uwave{-1.035}}
& \makecell[c]{-0.899}
& \makecell[c]{-0.872}
& \makecell[c]{-0.911}
& \makecell[c]{-0.895}
& \makecell[c]{-0.346$\circ$}
& \makecell[c]{\underline{-1.313}}
& \makecell[c]{\textbf{-0.199}$\bullet$}
& \makecell[c]{0.613}
& \makecell[c]{0.653}
& \makecell[c]{\underline{0.438}}
& \makecell[c]{\uwave{0.546}}
& \makecell[c]{0.576}
& \makecell[c]{0.520}
& \makecell[c]{0.696}
& \makecell[c]{\textbf{1.260}$\bullet$}
& \makecell[c]{0.806$\circ$}
\\

Musae-G
& \makecell[c]{-0.118}
& \makecell[c]{-0.120}
& \makecell[c]{\underline{-0.693}}
& \makecell[c]{-0.110}
& \makecell[c]{-0.144}
& \makecell[c]{-0.220}
& \makecell[c]{-0.030$\circ$}
& \makecell[c]{\uwave{-0.308}}
& \makecell[c]{\textbf{-0.005}$\bullet$}
& \makecell[c]{\underline{0.112}}
& \makecell[c]{0.119}
& \makecell[c]{\textbf{0.527}$\bullet$}
& \makecell[c]{\uwave{0.118}}
& \makecell[c]{0.126}
& \makecell[c]{0.126}
& \makecell[c]{0.247}
& \makecell[c]{0.366$\circ$}
& \makecell[c]{0.213}
\\

\midrule
Amazon-C
& \makecell[c]{-0.129}
& \makecell[c]{-0.126}
& \makecell[c]{\underline{-0.350}}
& \makecell[c]{-0.134}
& \makecell[c]{-0.144}
& \makecell[c]{-0.175}
& \makecell[c]{-0.049$\circ$}
& \makecell[c]{\uwave{-0.298}}
& \makecell[c]{\textbf{-0.031}$\bullet$}
& \makecell[c]{0.094}
& \makecell[c]{0.095}
& \makecell[c]{0.258$\circ$}
& \makecell[c]{0.089}
& \makecell[c]{\uwave{0.087}}
& \makecell[c]{\underline{0.061}}
& \makecell[c]{0.201}
& \makecell[c]{\textbf{0.312}$\bullet$}
& \makecell[c]{0.215}
\\

Amazon-P
& \makecell[c]{-0.163}
& \makecell[c]{-0.180}
& \makecell[c]{\underline{-0.458}}
& \makecell[c]{-0.175}
& \makecell[c]{-0.203}
& \makecell[c]{-0.231}
& \makecell[c]{-0.058$\circ$}
& \makecell[c]{\uwave{-0.339}}
& \makecell[c]{\textbf{-0.034}$\bullet$}
& \makecell[c]{0.122}
& \makecell[c]{0.132}
& \makecell[c]{0.315$\circ$}
& \makecell[c]{0.123}
& \makecell[c]{\uwave{0.111}}
& \makecell[c]{\underline{0.090}}
& \makecell[c]{0.257}
& \makecell[c]{\textbf{0.377}$\bullet$}
& \makecell[c]{0.277}
\\

\midrule
Coauthor-C
& \makecell[c]{-0.216}
& \makecell[c]{-0.243}
& \makecell[c]{\underline{-0.745}}
& \makecell[c]{-0.264}
& \makecell[c]{-0.245}
& \makecell[c]{-0.341}
& \makecell[c]{-0.097$\circ$}
& \makecell[c]{\uwave{-0.411}}
& \makecell[c]{\textbf{-0.038}$\bullet$}
& \makecell[c]{\underline{0.183}}
& \makecell[c]{0.205}
& \makecell[c]{\textbf{0.568}$\bullet$}
& \makecell[c]{0.214}
& \makecell[c]{\uwave{0.184}}
& \makecell[c]{\uwave{0.184}}
& \makecell[c]{0.268}
& \makecell[c]{0.457$\circ$}
& \makecell[c]{0.263}
\\

Coauthor-P
& \makecell[c]{-0.146}
& \makecell[c]{-0.144}
& \makecell[c]{\underline{-0.720}}
& \makecell[c]{-0.220}
& \makecell[c]{-0.159}
& \makecell[c]{-0.295}
& \makecell[c]{-0.057$\circ$}
& \makecell[c]{\uwave{-0.314}}
& \makecell[c]{\textbf{-0.035}$\bullet$}
& \makecell[c]{\underline{0.133}}
& \makecell[c]{0.141}
& \makecell[c]{\textbf{0.534}$\bullet$}
& \makecell[c]{0.149}
& \makecell[c]{\uwave{0.138}}
& \makecell[c]{0.167}
& \makecell[c]{0.208}
& \makecell[c]{0.367$\circ$}
& \makecell[c]{0.206}
\\
\bottomrule
\end{tabular}
\label{tab:overall}
\end{table*}

\subsection{Datasets and Baselines}
\noindent\textbf{Datasets and experimental settings}.
We drew real-world datasets from four applications for the node classification task.
The dataset details are provided in the supplement.
\begin{itemize}[leftmargin=*]
    \item 
In citation networks (Citeseer, Cora, PubMed)~\cite{kipf2017gcn}, each paper has bag-of-words features, and the goal is to predict the research area of each paper.
\item
We adopt Musae-Facebook (Musae-F) and Musae-Github (Musae-G)~\cite{rozemberczki2019multi} from social networks.
Nodes represent official Facebook pages (or Github developers), and edges are mutual likes (or followers) between nodes.
Node features are extracted from site descriptions
(or developer's location, 
repositories starred, employer).
\item
Amazon-Computer (Amazon-C) and Amazon-Photo (Amazon-P)~\cite{shchur2018pitfalls} are segments of the Amazon co-purchase graph,
where nodes represent goods, edges indicate that two goods are frequently bought together, and node features are the bag-of-words representation of product reviews.
\item
Coauthor-Computer and Coauthor-Physics are co-authorship graphs based on the Microsoft Academic Graph from the KDD Cup 2016.
We represent authors as nodes,
that are connected by an edge if they co-authored a paper~\cite{shchur2018pitfalls}.
Node features represent paper keywords for each author’s papers.
\end{itemize}
We randomly divide each graph into three portions with a ratio of \textit{training : validation : test = 50 : 20 : 30}.
The GNN is trained on the training set and
all explanation methods are evaluated on the test set.

\noindent\textbf{Baselines}.
We adopt the following baselines that generate subgraph explanations.
Except the baseline Shapley,
all baselines compute the weights of edges in the neighborhood of the target node $v_i$.
The explanation $G_i$ is generated by iteratively adding edges adjacent to the current subgraph until $C$ nodes are included in $G_i$.
The edges with higher weights will be considered first.
The counterfactual $\tilde{G}_i$ of the baselines are generated in the same way as GNN-MOExp by trying different enumerated subgraphs.
$\nu(G_i)$ and $\mu(G_i, \tilde{G}_i)$ of $G_i$ for each baseline are calculated by Eq. (\ref{eq:sim}) and Eq. (\ref{eq:cfr}).
We describe the details of the baseline:

\begin{itemize}[leftmargin=*]
    \item 
\noindent\textbf{Random} (RND) assigns random weights to edges.
    \item
\noindent\textbf{Embedding} (EMB) uses DeepWalk~\cite{Bryan2014deepwalk} to embed the nodes,
and the weight of an edge is calculated based on the cosine similarity between the embeddings of two nodes.

    \item
\noindent\textbf{Gradient} (Grad)~\cite{baldassarre2019explainability} use the magnitudes of gradients of GNN output w.r.t. edges to find salient subgraphs.

    \item
\noindent\textbf{GAT}~\cite{velivckovic2018gat} learns attention weights over neighbors of any node for message aggregation to predict the output of the GNN on $v_i$,
and the attention weights on the edges are extracted as edge weights.

    \item
\noindent\textbf{GNNExplainer} (GNNExp)~\cite{ying2019gnn}
learns to mask edges so that the masked graph maximally preserve the predictions of $\mathbf{y}_i$,
and the mask matrix provides the edge weights.

    \item
\noindent\textbf{PGExplainer} (PGExp)~\cite{luo2020parameterized}
trains a deep neural network to parameterize the generation of explanations.
The subgraphs generated by the explainer are evaluated.

    \item
\noindent\textbf{Shapley}
picks $\Delta$ and $G_i\in \mathcal{ S}(\Delta; G)$, defined in Eq. (\ref{eq:sv}), with the highest counterfactual relevance, and use the selected $\Delta$ and $G_i$ to generate the counterfactual $\tilde{G}_i$.

    \item
\noindent\textbf{GNN-MOExp-b} (MOEB) is similar to GNN-MOExp, 
while the strategy is to select explanations that are most balanced in both metrics.

\end{itemize}

\subsection{Quantitative Results}
\begin{figure}[t]
    \centering
\begin{minipage}{.24\textwidth}
    \includegraphics[width=\textwidth]{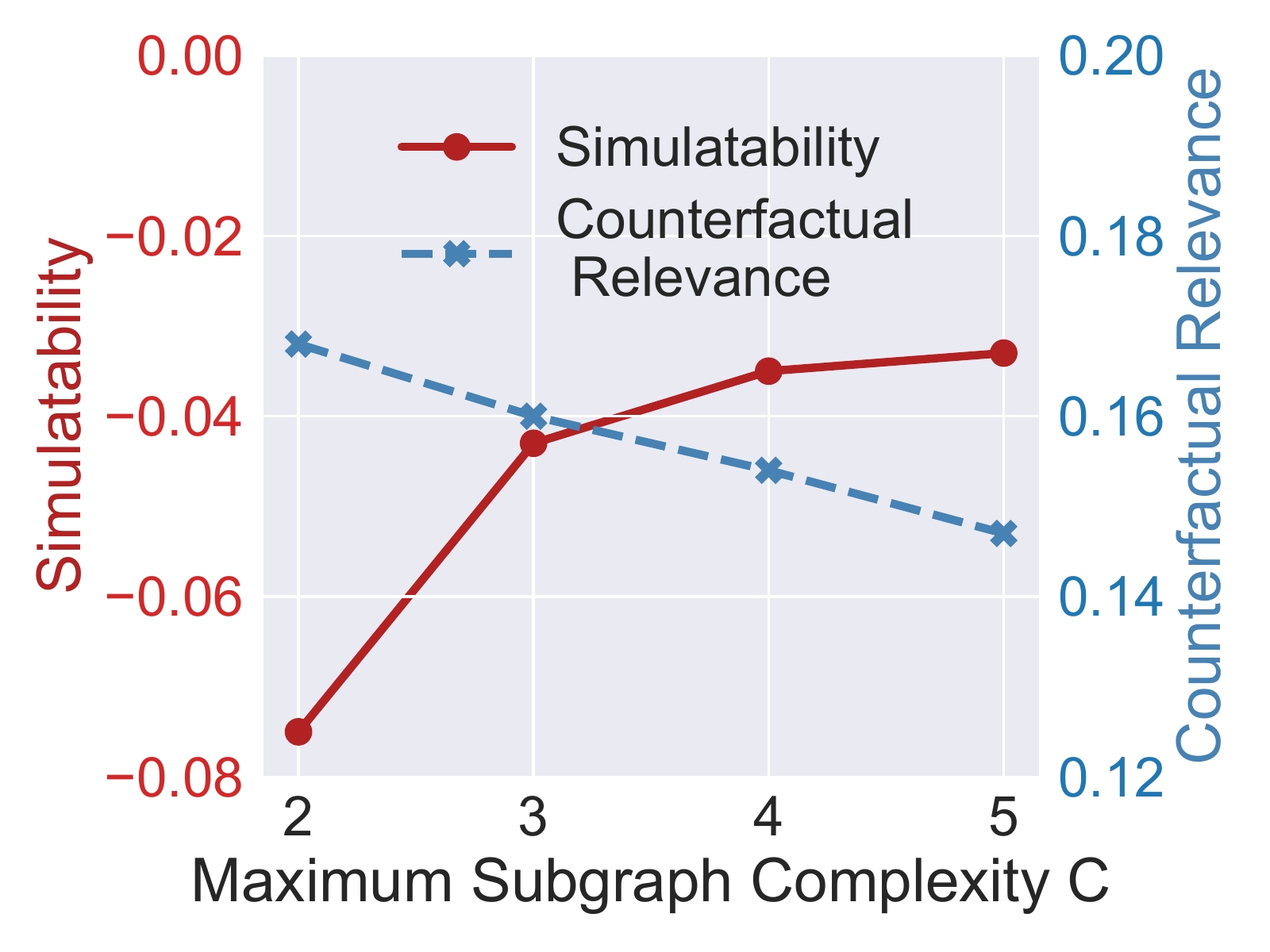}
\end{minipage}%
\begin{minipage}{.24\textwidth}
    \includegraphics[width=\textwidth]{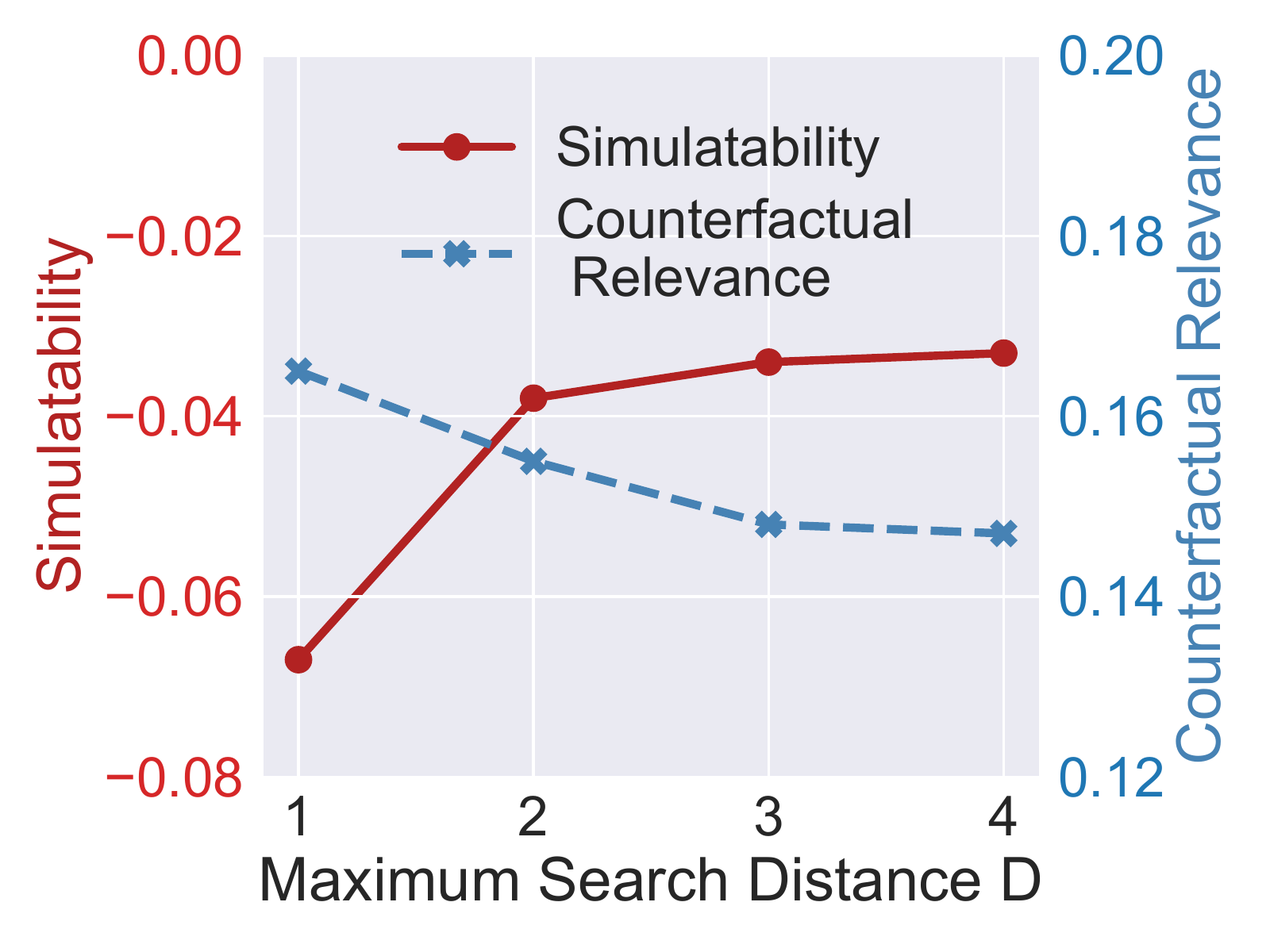}
\end{minipage}%
\caption{\small Parameters sensitivity of maximum subgraph complexity $C$ and maximum search distance $D$ on Citeseer.
}
\label{fig:sensitivity_analysis}
\end{figure}

Average simulatability and counterfactual relevance across all test nodes are reported in Table \ref{tab:overall}.
We conclude that:
\begin{itemize}[leftmargin=*]
    \item Gradient does not perform badly in counterfactual relevance (best in three datasets and second places in 2 datasets),
    but it performs worst or the second-worst in simulatability except the Musae-F dataset.
    That's because the gradients indicate the most effective perturbations of the edges to change a prediction.
    However, these edges do not constitute a graph to maximally preserve the GNN prediction.
    Based on the human study, Grad should be first excluded.
    \item GAT, GNNExplainer, and PGExp are outperformed by GNN-MOExp in both metrics on all datasets.
    Clearly, these baselines do not explicitly optimize both objectives.
    \item MOEB has the worst or second-worst simulatability on the latter 6 datasets, though it is the runner-up on the first three.
    Based on the human study, MOEB is not guaranteed to generate explanations that will likely be accepted.
    \item Shapley has the best counterfactual relevance on the first three datasets, with GNN-MOExp as the runner-up.
    On the remaining 6 datasets, GNN-MOExp outperforms or is close to Shapley.
    On simulatability, GNN-MOExp outperforms Shapley on all datasets.
    \item GNN-MOExp is the best in simulatability on all baselines on all datasets,
    and is frequently outperforming or competitive with the feasible runner-ups (Grad and MOEB are not feasible due to their low simulatability).
\end{itemize}

\begin{figure}[t]
    \centering
\begin{minipage}{.24\textwidth}
    \includegraphics[width=\textwidth]{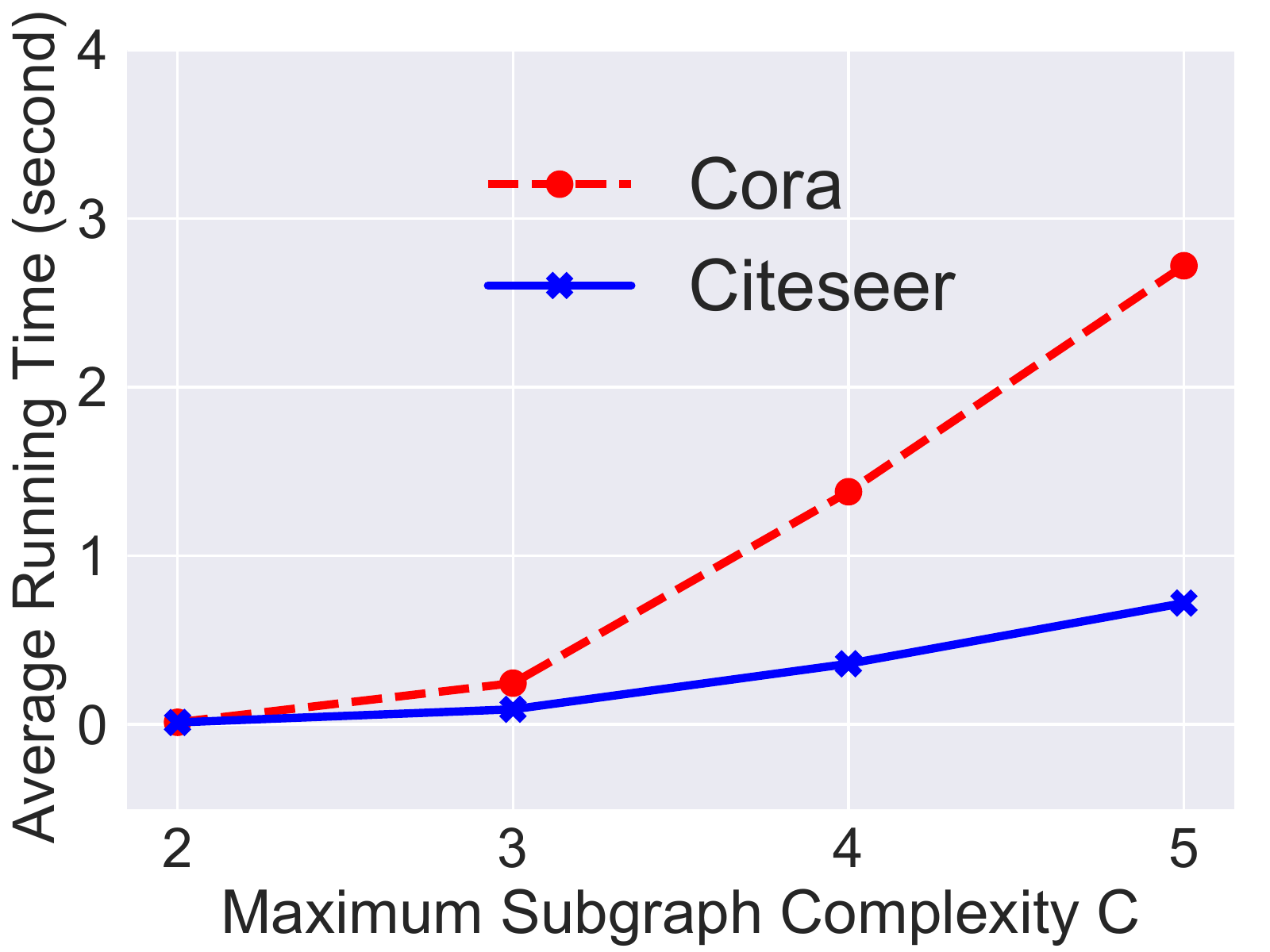}
\end{minipage}%
\begin{minipage}{.24\textwidth}
    \includegraphics[width=\textwidth]{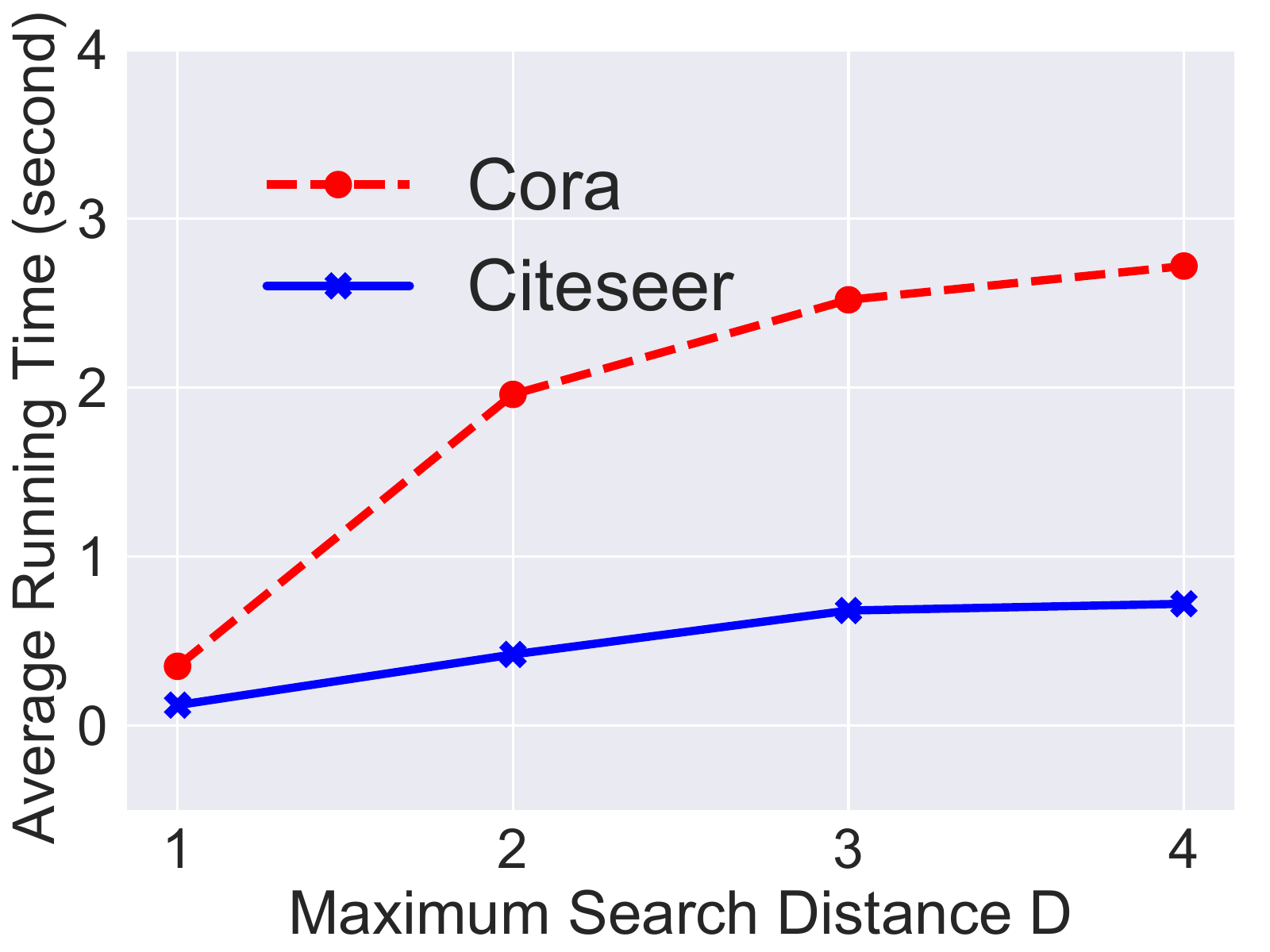}
\end{minipage}%
\caption{\small Average Running time for each node of maximum search distance $D$ and maximum subgraph complexity $C$ on Cora and Citeseer.
}
\label{fig:running_time}
\end{figure}

\noindent\textbf{Parameters Sensitivity}.
We search subgraphs of $C$ nodes involving vertices that are $D$ hops away from the target node (by default $D=L$, the depth of the target GNN).
The sensitivity analyses of these parameters are shown in Fig. \ref{fig:sensitivity_analysis}.
We can see that the performance of simulatability becomes better as the parameters increase,
while the performance of counterfactual relevance becomes lower.
We let $C\leq 5$ since
large explaining subgraphs go against explanation simplicity and simulatability.
Since $L$ is usually small ($=2$ in our experiments) to avoid over-smoothing~\cite{Li2018DeeperII}, we can see the performance level off when $D\geq 2$.

One bottleneck of applying GNN-MOExp to real-world graphs is its running time~\cite{Rudin2019GloballyConsistentRS}.
In Fig.~\ref{fig:running_time} we can see that the running time increases as the search space grows with $D$ and $C$.
However, on average, enumerating and evaluating all acyclic and connected subgraphs of a target node on Cora and Citeseer with some very high node degrees,
take no more than 3 seconds on a commodity computer.
With an incremental implementation,
a newly added edge only leads to enumerating new subgraphs containing the new edge.
Given the reasonable running time,
the capability of guaranteeing Pareto optimality and simultaneous high simulatability and counterfactual relevance is a unique advantage that gradient-based methods do not have.  
Explaining GNN with a quality guarantee
is a must-have when GNN is used in user-centric applications, such as graph-based recommendation systems~\cite{Ying2018}.

\subsection{Robustness and Sanity check}
\label{sec:exp_sanity_check_robustness}
We design two ways to perturb GNN predictions.
We can link an existing vertex to the target node $v_i$ and
add a message $\mathbf{m}^{(L)}_{j'i}$ to Eq. (\ref{eq:aggregate}) at the last layer of GNN:
\begin{equation}
\label{eq:aggregate_perturbed}
\tilde{\mathbf{a}}^{(L)}_i = \textnormal{AGG}\left(\left\{\mathbf{m}^{(L)}_{ji}|v_j \in \mathcal{N}(v_i)\right\}\cup\left\{\mathbf{m}^{(L)}_{j'i}\right\}\right),
\end{equation}
where
$\tilde{\mathbf{a}}^{(L)}_i$ is the perturbed activation.
We measure the strength of the perturbation caused by $\mathbf{m}^{(L)}_{j'i}$ using 
\begin{equation}
d_m\left(\mathbf{m}^{(L)}_{j'i}, \boldsymbol{\theta}_{y}^{(L)}\right) = -\cos\left(\mathbf{m}^{(L)}_{j'i}, \boldsymbol{\theta}_{y}^{(L)}\right),
\end{equation}
where $y$ is the predicted class of $v_i$ before the perturbing edge is added.
Second,
we randomize the GNN parameters $\boldsymbol{\theta}^{(L)}$ of layer $L$,
which is the last layer of GNN.
We measure the perturbation strength using Euclidean distance between the original parameters $\boldsymbol{\theta}^{(L)}$ and the perturbed parameters $\tilde{\boldsymbol{\theta}}^{(L)}$
\begin{equation}
    d_\theta(\boldsymbol{\theta}^{(L)}, \tilde{\boldsymbol{\theta}}^{(L)})=\|\boldsymbol{\theta}^{(L)} - \tilde{\boldsymbol{\theta}}^{(L)}\|_2.
\end{equation}

Given a perturbation, we need to measure the change in the explaining subgraph $G_i$ of $\mathbf{y}_i$.
Let the explaining subgraphs after the perturbation be denoted by  $\tilde{G}_i$.
We measure the average distance between two explaining subgraph $d_e(G_i,\tilde{G}_i)$,
where $d_e$ is Jaccard distance between two vertex sets.

\begin{figure}[t]
    \centering
\begin{minipage}{.24\textwidth}
    \includegraphics[width=\textwidth]{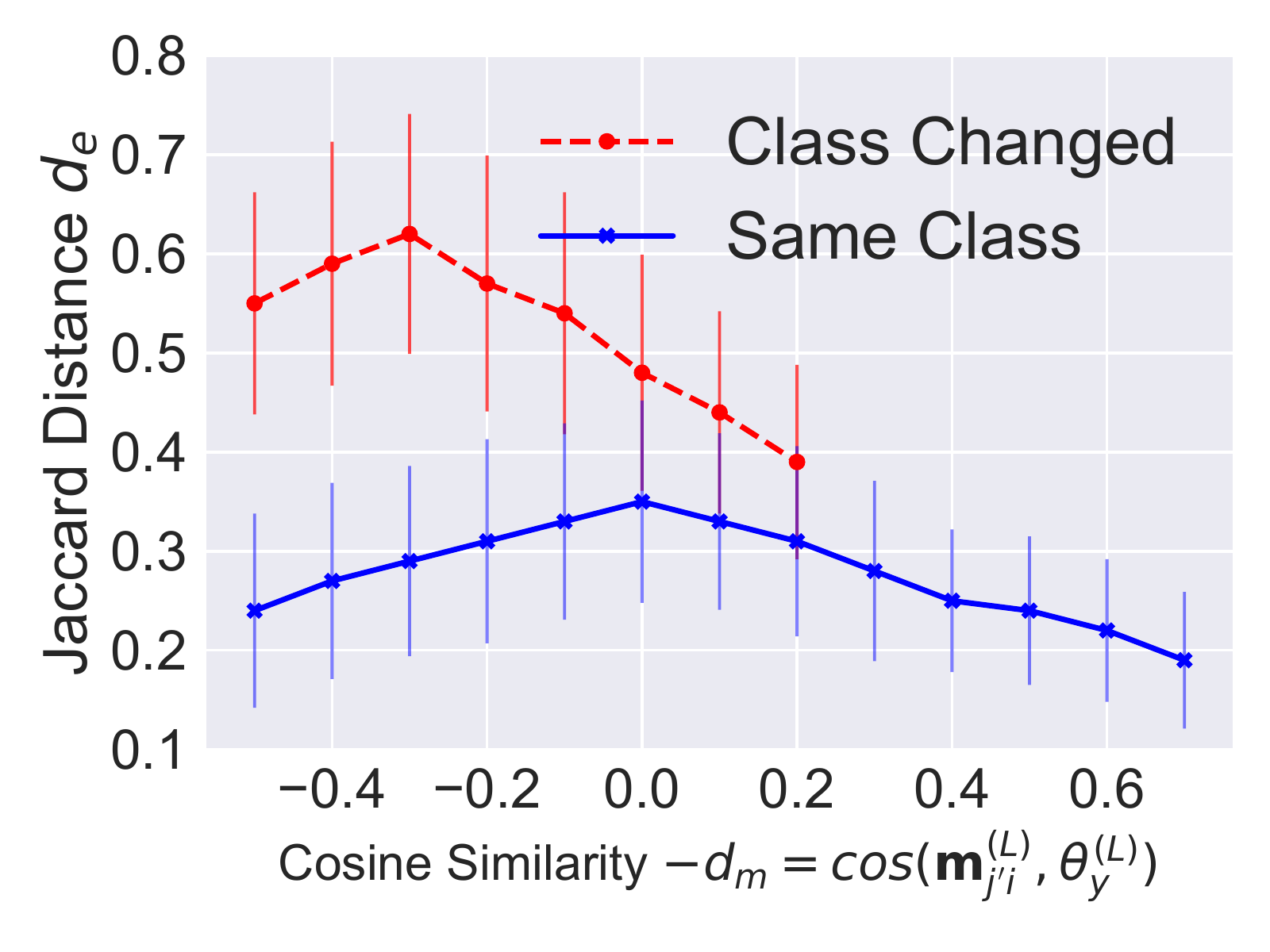}
\end{minipage}%
\begin{minipage}{.23\textwidth}
    \includegraphics[width=\textwidth]{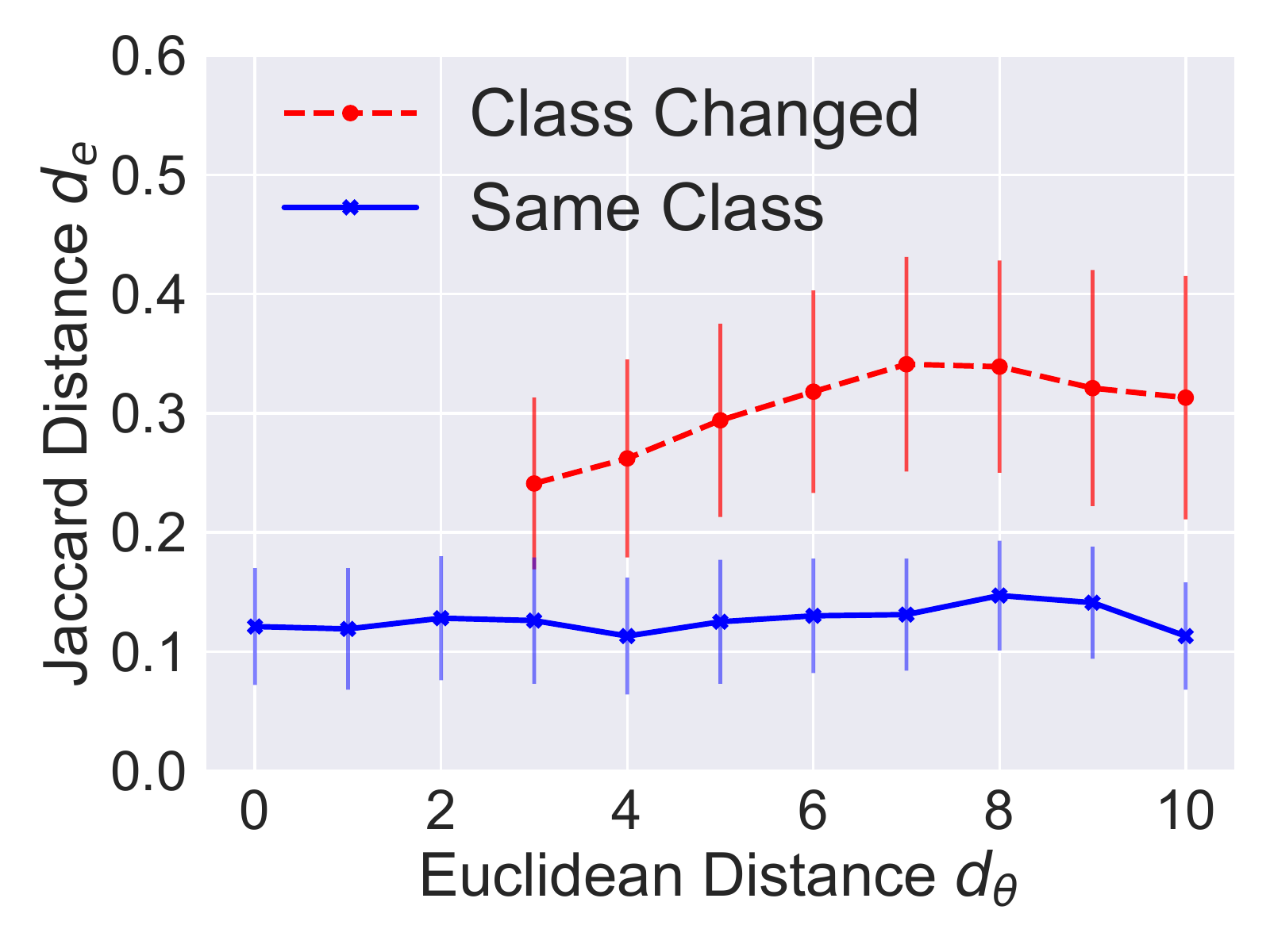}
\end{minipage}%
\caption{\small
Robustness and sanity check.
\textit{Left}: Jaccard distance $d_e$ changes due to perturbed messages.
\textit{Right}: Jaccard distance $d_e$ changes due to perturbed GNN parameters.
}
\label{fig:robustness}
\end{figure}

From Fig.~\ref{fig:robustness}, we can observe that the subgraph explanations found by our method pass the sanity check.
We have the following observations.
i) There is no change in the predicted class by the target GNN when the perturbing message $\mathbf{m}^{(L)}_{j'i}$ is aligned with $\theta_y^{(L)}$ (high cosine similarity) or the perturbing distance is small, and predictions start to change when the perturbations are sufficiently strong.
ii) The Jaccard distance $d_e$ between two optimal explaining subgraphs becomes larger as predicted class changes, demonstrated by the red curves on top of the blue curve.
iii)
Interestingly, on the left,
even when there is no change in the predicted class,
$d_e$ first increases as cosine similarity $d_m$ decreases to 0 ($\mathbf{m}_{j^\prime i}^{(L)}$ is orthogonal to $\boldsymbol{\theta}_y^{(L)}$),
and then decrease again when $d_m$ further decreases to negative values ($\mathbf{m}_{j^\prime i}^{(L)}$ is in the opposite direction of $\boldsymbol{\theta}_y^{(L)}$).
We conjecture that the edge ($j^\prime, i$) is added to the explaining graph in the former situation, while some message cancel out the opposite $\mathbf{m}_{j^\prime i}^{(L)}$ in the latter case (though there may not always be such a canceling message).
The explanations are more robust to perturbing $\boldsymbol{\theta}^{(L)}$ as the $d_e$ remains low if predictions remain the same (right figure).
The explanations are more sensitive to perturbing incoming messages (left figure).
In such cases, on average less than two edges are perturbed in the explaining subgraphs.

\subsection{Reproducibility checklist}
We adopt a Graph Convolutional Network (GCN) model \cite{kipf2017gcn}  as the explained target model,
with two hidden layers ($L=2$), each with 16 neurons.
The dimension of the input layer is the number of input features of the nodes,
and the dimension of the output layer is the number of classes.
We adopt the cross-entropy loss function and the Adam optimizer for training the GNN,
while the learning rate is set to be 0.01.
We set the maximal training iterations to 500,
and apply the early-stop strategy when training.

As for the proposed GNN-MOExp,
there are two hyper-parameters.
We set the maximum search distance $D=2$,
which is equal to the depth of GCN,
and we set the maximum subgraph complexity $C=4$,
considering both the effectiveness and the explanation simplicity.
\section{Related Work}

\textbf{Explainable ML}.
The simulatability and counterfactual relevance are two major metrics for evaluating explanations, but their interactions and how humans perceive them are not clear. 
In ~\cite{Lundberg2017UnifiedApproach} and ~\cite{Shrikumar2017DEEPLIFT},
they provide a prediction explanation framework based on Shapley values which encompasses LIME as a special case. 
Two algorithms with linear complexity for feature importance scoring are developed in \cite{chen2018shapley}.
In \cite{Ghorbani2019DATAshapley} and \cite{Ancona2019DNNshapley},
they approximate Shapley values for deep networks via sampling.
The methods proposed in 
~\cite{darwiche2003differential,chan2005sensitivity}
use gradients to find salient subgraphs to explain the inference on PGM,
but not for GNNs~\cite{kipf2017gcn,hamilton2017graphsage,velivckovic2018gat}.
~\cite{ying2019gnn} explains arbitrary graph neural networks using a simplified model.
~\cite{baldassarre2019explainability} studies the influence of the change of inputs on outputs of GNN models with gradient-based and decomposition-based methods.
\textit{Stochastic} explaining subgraph search have been proposed~\cite{yuan2020xgnn,Vu2020PGMExplainerPG,yuan2021explainability} using reinforcement learning and hill-climbing. 
In \cite{yuan2021explainability},
Monte Carlo search is used for exploration.

\noindent\textbf{Causal Inference and Counterfactual Reasoning}.
\cite{guo2020survey} introduces both traditional and advanced methods in learning causal effect and causal relations.
In \cite{guo2020learning}, they discover the unknown confounders from observed data,
by learning representations of confounders using GNN.
We identify confounders on the computational graph of GNN.

\noindent\textbf{Robustness and sensitivity}.
Explanation robustness and sensitivity are two desired properties and have been mostly studied on images \cite{Adebayo2018,Ghorbani2017,Zhang2018,Yeh2019,Pruthi2019} and texts \cite{Pruthi2019},
but none on graphs.
The differential geometry formulation of manipulability of gradient-based explanations in~\cite{Adebayo2018} assumes that the input is a vector (image) that lies on a low-dimensional manifold.
For GNN, a decision of a node depends not only on its feature vectors, but also on the messages from neighboring nodes.
On graphs, the only relevant study is~\cite{wiltschko2020},
and the proposed method differs from~\cite{wiltschko2020} in explanation generation (subgraph search vs. gradient-based) and evaluation metrics (output explanation changes vs. attribution accuracy changes).

\section{Conclusion and future work}
We proposed to find multi-objective explanations for Graph Neural Networks,
with two objectives, simulatability and counterfactual relevance, to be satisfied.
The human study showed that the two explanation objectives can represent the perceived quality of explanations based on two different cognitive processes (quick screening \textit{vs}. effortful deliberation),
and they jointly influence and predict explanation acceptance by humans.
We proposed to maximize the two objectives by subgraph enumeration and ranking-based optimization to produce Pareto optimal explanations that fulfill both objectives.
We showed that gradient-based GNN explanations are not robust
against the rotation of incoming messages to the target nodes,
while GNN-MOExp can reliably output quality explanations.
Extensive experiments on 9 graph datasets from 4 applications demonstrated superior performance in simulatability, counterfactual relevance, robustness, and sensitivity. 

\section*{Acknowledgement}
{\small
Chao and Sihong were supported in part by the National Science Foundation under Grants NSF IIS-1909879, NSF CNS-1931042, and  NSF IIS-2008155.
Any opinions, findings, conclusions, or recommendations expressed in this document are those of the author(s) and should not be interpreted as the views of any U.S. Government.
Yifei, Yazheng, and Xi were supported by Natural Science Foundation of
China (No.61976026) and 111 Project (B18008).
}
\normalem
\bibliographystyle{plain}
\bibliography{paper.bib}
\end{document}